\documentclass[10pt,journal,compsoc]{IEEEtran}

\usepackage{
amsfonts,
cite,
pifont,
graphicx,
multirow,
booktabs,
amsmath,
amssymb,
hhline,
url,
hyperref}
\usepackage[ruled]{algorithm2e}
\usepackage[table,xcdraw]{xcolor}
\usepackage{mathtools} 

\newcommand\shline{\specialrule{0.8pt}{0pt}{0pt}}

\usepackage{xcolor}

\usepackage{amsthm}  
\newtheorem{theorem}{Theorem}
\newtheorem{lemma}{Lemma}
\newtheorem{corollary}{Corollary}

\definecolor{darkblue}{rgb}{0, 0, 0.6}
\definecolor{darkgreen}{rgb}{0, 0.7, 0}
\definecolor{darkred}{rgb}{0.8, 0, 0}

\definecolor{mygray}{gray}{0.98}
\definecolor{mygray2}{gray}{0.75}

\let\oldcite\cite
\renewcommand{\cite}[1]{\textcolor{darkblue}{\oldcite{#1}}}

\newcommand{\eg}{\textit{e.g.}}
\newcommand{\ie}{\textit{i.e.}}
\newcommand{\best}[1]{\textbf{\textcolor{red}{#1}}}

\newcommand*\colorcheck{%
  \expandafter\newcommand\csname greencheck\endcsname{\textcolor{darkgreen}{\ding{52}}}%
}
\newcommand*\colorcross{%
  \expandafter\newcommand\csname redcross\endcsname{\textcolor{darkred}{\ding{56}}}%
}
\colorcheck
\colorcross
\newcommand{\etal}{\textit{et al.}}
\newcommand{\aka}{\textit{a.k.a.}}

\newcounter{proofsection}
\newcommand{\proofsection}[1]{%
    \refstepcounter{proofsection}%
    \section*{#1}%
}

\begin{document}
\title{Plug In, Grade Right: Psychology-Inspired AGIQA}

\author{Zhicheng Liao, Baoliang Chen,  Hanwei Zhu, Lingyu Zhu, Shiqi Wang,~\IEEEmembership{Senior~Member,~IEEE}, and Weisi Lin,~\IEEEmembership{Fellow,~IEEE}
\IEEEcompsocitemizethanks{
\IEEEcompsocthanksitem Zhicheng Liao is with the School of Computer Science, South China Normal University. E-mail: zcliao@m.scnu.edu.cn.
\IEEEcompsocthanksitem Baoliang Chen is with the School of Computer Science, South China Normal University, and the School of Computer Science and Engineering, Nanyang Technological University. E-mail: blchen6-c@my.cityu.edu.hk.

\IEEEcompsocthanksitem  Lingyu Zhu, and Shiqi Wang are with the School of Computer Science, City University of Hong Kong. Emails: lingyzhu-c@my.cityu.edu.hk, shiqwang@cityu.edu.hk.
\IEEEcompsocthanksitem Hanwei Zhu, and Weisi Lin are with the School of Computer Science and Engineering, Nanyang Technological University. Emails: hanwei.zhu@ntu.edu.sg, wslin@ntu.edu.sg.
}}

\IEEEtitleabstractindextext{
\begin{abstract} 

AI-generated image quality assessment (AGIQA) has garnered significant attention in recent years due to its critical role in quality control of AGIs. Existing AGIQA models typically estimate image quality by measuring and aggregating the similarities between image embeddings and text embeddings derived from multi-grade quality descriptions (\eg, ``a photo with excellent/good/poor quality"). Although effective, we observe that such similarity distributions across grades usually exhibit multimodal patterns. For instance, an image embedding may show high similarity to both ``excellent" and ``poor" grade descriptions while deviating from the ``good" one. We refer to this phenomenon as \textbf{\textit{``semantic drift''}}, where semantic inconsistencies between text embeddings and their intended descriptions undermine the reliability of text–image shared-space learning.
To mitigate this issue, we draw inspiration from psychometrics and propose an improved Graded Response Model (GRM) for AGIQA.  The GRM is a classical assessment model that categorizes a subject's ability across grades using test items with various difficulty levels. This paradigm aligns remarkably well with human quality rating, where image quality can be interpreted as an image’s ability to meet various quality grades. Building on this philosophy, we design a two-branch quality grading module: one branch estimates image ability while the other constructs multiple difficulty levels. To ensure monotonicity in difficulty levels, we further model difficulty generation in an arithmetic manner, which inherently enforces a unimodal and interpretable quality distribution.
Our \textbf{A}rithmetic \textbf{G}RM based \textbf{Q}uality \textbf{G}rading (\textbf{AGQG}) module enjoys a plug-and-play advantage, consistently improving performance when integrated into various state-of-the-art AGIQA frameworks. 
Moreover, it also generalizes effectively to both natural and screen content image quality assessment, revealing its potential as a key component in future IQA models. The code will be publicly available.

\end{abstract}

}

\maketitle

\section{Introduction}
\IEEEPARstart{W}{e} are undoubtedly living in an era where AI-generated images (AGIs) are produced and consumed at an unprecedented scale, spanning entertainment, education, and artistic creation\cite{chen2024learning,tian2025ai,zhu2023moviefactory}. Despite their growing ubiquity, the quality of AGIs remains highly inconsistent, often exhibiting both perceptual distortions and misalignment with their intended textual prompts\cite{zhang2023perceptual,li2023agiqa,wang2023aigciqa2023}. Such variability poses significant challenges for both users and developers, highlighting the need for a reliable and objective assessment of AGI quality.

Unlike traditional natural image quality assessment (IQA) models that focus mainly on perceptual fidelity, evaluating AGI quality (AGIQA) requires a shift toward multi-modal quality modeling, considering both visual content and textual prompts. Existing AGIQA methods typically follow a \textit{\textbf{``quality prediction by quality grading''}} framework, which has been shown to outperform direct quality regression\cite{tang2025clip,peng2024aigc,talebi2018nima}. For example, IPCE\cite{peng2024aigc} integrates image and text modalities to evaluate the prompt-image alignment by predefining five quality grade descriptions:\\

\noindent $\text {``A photo that }\{v\} \text { matches 'prompt'", } \text{with} \\  v \in[\text { ``badly", ``poorly", ``fairly", ``well", ``perfectly"] }.$ \\

\noindent A text encoder and an image encoder are then adopted to extract the embeddings of the descriptions and the image. The alignment score thus can be estimated by a weighted summation of the relative similarities between the image embedding and the five text embeddings. Analogously, two antonym templates: \{``good photo." and ``bad photo."\}  and six adjective  templates: \{``terrible", ``bad", ``poor", ``average", ``good", and ``perfect"\} are respectively adopted in  CLIP-IQA\cite{wang2023exploring} and CLIP-AGIQA\cite{tang2025clip} for the quality-grade categorization. 

\begin{figure*}[!t]
\centering
\includegraphics[width=1.0\textwidth]{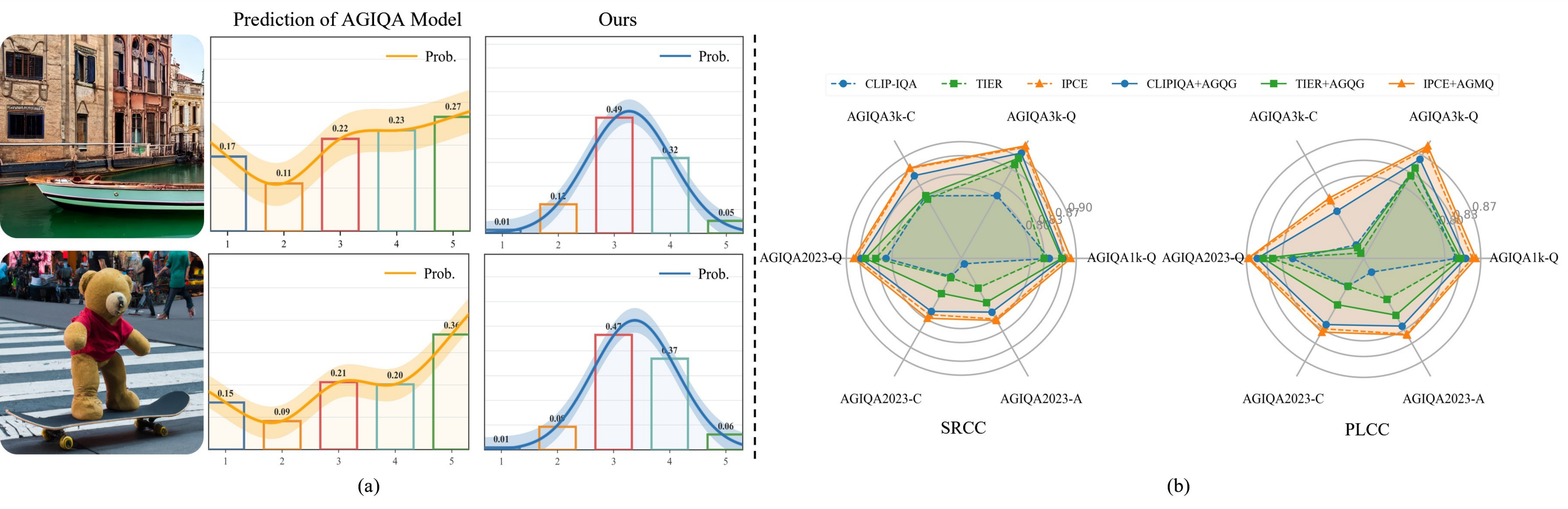}
\vspace{-25pt}
\caption{Illustration of the \textit{``semantic drift"} phenomenon in existing AGIQA models. 
(a) Comparison of quality grade predictions between IPCE and our enhanced model. 
(b) Our AGQG module consistently improves existing AGIQA models (CLIP-IQA\cite{wang2023exploring}, TIER\cite{yuan2024tier}, IPCE\cite{peng2024aigc}) across six diverse datasets spanning perceptual quality---\textit{Q}, prompt consistency---\textit{C}, and perceptual authenticity---\textit{A} dimensions.
}

\label{fig_1}
\end{figure*}
Despite their effectiveness, we observe that the semantics encoded in text embeddings are not always consistent with their corresponding quality descriptions. As illustrated in Fig.~\ref{fig_1}(a), existing AGIQA models usually produce multi-peak prediction distributions, where the image embedding exhibits higher similarity with both ``excellent” (score=5) and ``poor” (score=1) descriptions than other intermediate quality grades.
We attribute this counterintuitive behavior to  \textbf{\textit{Semantic Drift}}, where the semantics of the embedding deviates from its intended textual description. The semantic drift indicates that an unshared space is learned for the text and image modalities, ultimately leading to an inaccurate quality prediction.

To mitigate the semantic drift, we propose a new paradigm for quality-grade categorization. Specifically, we reinvigorate the classical Graded Response Model (GRM) from psychometrics\cite{samejima1969estimation, uto2015item} and integrate it into the quality-grade categorization design. The GRM is a type of Item Response Theory (IRT) model, 
estimating the probability that a respondent will select a particular grade based on their latent traits, such as ability or severity.
For example, in an educational assessment, given the ability level of a subject and a set of test questions with increasing difficulty, the GRM estimates the probability of correctly answering each question by comparing the ability level with the thresholds that define question difficulty.
The categorization philosophy of the GRM naturally aligns with human quality rating, wherein image quality can be likened to a specific subject's ability, and different quality grades mirror the difficulty of ordered response categories.  Following this vein, we introduce an \textbf{A}rithmetic \textbf{G}RM based \textbf{Q}uality \textbf{G}rading (\textbf{AGQG}) module to replace the text-image similarity measurement scheme in existing AGIQA models.

Our AGQG framework consists of two branches: the first branch generates the ability score of the input image, while the second constructs multi-grade difficulties.  To enforce an ordered structure for the difficulty levels, we enforce their thresholds in an arithmetic manner. Theoretically, this arithmetic constraint ensures a single-peak (unimodal) response probability distribution, aligning naturally with human rating patterns. Based on the estimated probabilities for each category, a weighted aggregation scheme is applied to derive the final quality score prediction. Notably,  our AGQG module is independent of image and text encoders, making it a plug-and-play module that can be seamlessly integrated into existing AGIQA models. As shown in Fig.~\ref{fig_1}(b), three AGIQA models with AGQG module obtain consistent performance improvements across all evaluation scenarios. Before delving into the details, we summarize our key contributions as follows:
\begin{itemize}
\item We identify the semantic drift phenomenon in existing AGIQA models, which leads to a multimodal probability distribution in quality-level categorization. This phenomenon highlights the unreliable learning of the text-image shared space, ultimately resulting in inaccurate quality predictions.

\item Inspired by psychological item response theory, we make the first attempt to introduce the GRM into quality-level categorization for AGIQA. Our AGQG module theoretically guarantees a single-peak probability distribution, ensuring consistency with human quality-level ratings. Moreover, its plug-and-play design enables flexible integration into existing AGIQA models for performance enhancement.

\item Extensive experiments on multiple AGIQA datasets and models demonstrate the strong generalization capability of our AGQG module. Furthermore, the consistent enhancement across different content types, including natural images and screen content images, highlights its potential as a fundamental component for future IQA designs.

\end{itemize}

\section{Related Work}
\label{sec:related}

\subsection{No-reference Natural Image Quality Assessment}
\label{subsec:iqa}
Traditional IQA  models are primarily designed for natural images, aiming to provide an objective measure of perceptual quality. Depending on the availability of reference information, IQA models can be classified into three categories: full-reference (FR), reduced-reference (RR), and no-reference (NR) methods. Among these, NR-IQA is especially valuable in practice due to its independence from reference images.
Early NR-IQA models predominantly leveraged natural scene statistics (NSS) to model the statistical regularities of high-quality images. For example, the Mean-Subtracted Contrast-Normalized (MSCN) coefficients were introduced in NIQE\cite{mittal2012making}, enabling quality prediction by measuring the deviations from natural distributions. The NSS characteristics were also widely explored in other transformed domains, including the wavelet domain\cite{moorthy2010two, tang2011learning}, gradient domain\cite{xue2014blind}, and discrete cosine transform (DCT) domain\cite{saad2012blind}. Inspired by the free-energy principle in neuroscience\cite{friston2006free, friston2010free},  the IQA can also be formed as a process of uncertainty estimation relative to internal visual predictions\cite{gu2013no, zhai2011psychovisual, chen2022no}. While these hand-crafted feature-based methods offered valuable insights, they remained limited in capturing the full complexity of the human visual system.

The advent of deep learning marked a paradigm shift in IQA, allowing models to learn quality representations in a data-driven manner\cite{kang2014convolutional, kang2015simultaneous}. Compared with the hand-crafted IQA models, the convolutional neural network (CNN) based models, such as DBCNN\cite{zhang2018blind}, RankIQA\cite{liu2017rankiqa}, FPR\cite{chen2022no}, and GraphIQA\cite{sun2022graphiqa}, have achieved notable improvements in quality prediction accuracy. To further enhance the model generalization across datasets and distortion types, more advanced learning strategies have been investigated, including mixed-dataset training\cite{feng2023learning, zhang2021uncertainty, zhang2023blind}, meta-learning\cite{zhu2020metaiqa}, continual learning\cite{zhang2024task}, and transfer learning\cite{chen2021no, chen2021learning}.
More recently, the emergence of large multimodal models (LMMs) with strong zero-shot reasoning capabilities has opened new frontiers in IQA research. Early studies, such as Q-Bench\cite{wu2023qbench} and DepictQA\cite{you2024descriptive}, introduced benchmark datasets to systematically evaluate LMMs' ability to assess low-level visual attributes. Subsequent works extended this line through multitask adaptation of CLIP\cite{zhang2023blind} and SFT-based approaches, including Q-Align\cite{wu2023align}, Compare2Score\cite{zhu2024adaptive}, DeQA-Score\cite{deqa_score}, Q-Insight\cite{li2025q}, and VisualQuality-R1\cite{wu2025visualquality}, which trained LMMs to produce quality scores, distributions, or descriptive judgments. This emerging paradigm shows great promise for aligning automated assessments with human perception across diverse IQA tasks.

\subsection{AI-Generated Image Quality Assessment}
\label{subsec:clip_based_iqa}
The rapid advancement of generative models, spanning from Generative Adversarial Networks (GANs)\cite{goodfellow2020generative} and autoregressive models\cite{ding2021cogview} to diffusion-based approaches\cite{rombach2022high}, has significantly propelled the development of AGIs. While these models have achieved impressive improvements in terms of resolution and visual fidelity, challenges such as structural inconsistencies, semantic misalignment, and loss of fine-grained details still severely limit their practical applicability. To control and optimize the quality of AGIs, AGIQA has emerged as a critical research topic in recent years. To facilitate AGIQA research, several datasets have been introduced for subjective quality evaluation, including AGIQA-1K\cite{zhang2023perceptual}, AGIQA-3K\cite{li2023agiqa}, and AIGCIQA2023\cite{wang2023aigciqa2023}.

Enabled by these datasets, a variety of AGIQA models have emerged, mainly focusing on two complementary evaluation dimensions: (1) Perceptual Quality Assessment, which evaluates the naturalness and realism of generated images, and (2) Prompt Alignment Assessment, which measures how well an image aligns with its corresponding text prompt. For perceptual quality, traditional metrics such as Inception Score (IS)\cite{salimans2016improved} and Fréchet Inception Distance (FID)\cite{hessel2021clipscore} remain widely adopted. Recently, Fang \etal\cite{fang2024pcqa} proposed a feature adapter and a feature mixer to model the interaction between prompt and visual features for AIGI perceptual quality inference. Fu \etal\cite{fu2024vision} employed learnable textual and visual prompts to enhance the extraction of quality-related cues. CLIP-IQA\cite{wang2023exploring} leveraged cosine similarity between image features and antonym-paired quality-level prompts, achieving superior performance over single-prompt-based quality inference.
For prompt alignment assessment, multi-level textual descriptions were explored to enhance evaluation granularity. For example, in TIER\cite{yuan2024tier}, Peng \etal.\cite{peng2024aigc} evaluated image alignment by computing cosine similarity between generated images and prompts of five quality grades, demonstrating superior performance compared to a direct quality regression. Analogously, CLIP-AGIQA\cite{tang2025clip} learns six quality category prompts to assess both perceptual and alignment quality. Xia et al.\cite{xia2024ai} introduced five-grade quality descriptions to capture both fine-grained and coarse-grained similarities between an image and its corresponding prompt, measuring perceptual quality and prompt alignment within a unified framework.

Despite these advancements, the existing text-image similarity based models remain heavily dependent on predefined quality-grade descriptions. This reliance introduces several challenges. \textbf{First, linguistic vagueness is inevitable;} for example, distinguishing between qualitative terms such as ``high quality” and ``good quality” when defining quality grades can be subjective and inconsistent. Even with extensive manual tuning, the designed prompts are not guaranteed to be optimal for specific downstream tasks. \textbf{Second, the semantic drift between textual descriptions and their corresponding latent embeddings, potentially leading to misalignment in quality assessment.} Therefore, it is in high demand to develop a new paradigm for image quality grading with the above limitations mitigated.

\section{Method}
\label{sec:method}

\begin{figure*}[!t]
\centering
\includegraphics[width=1.0\textwidth]{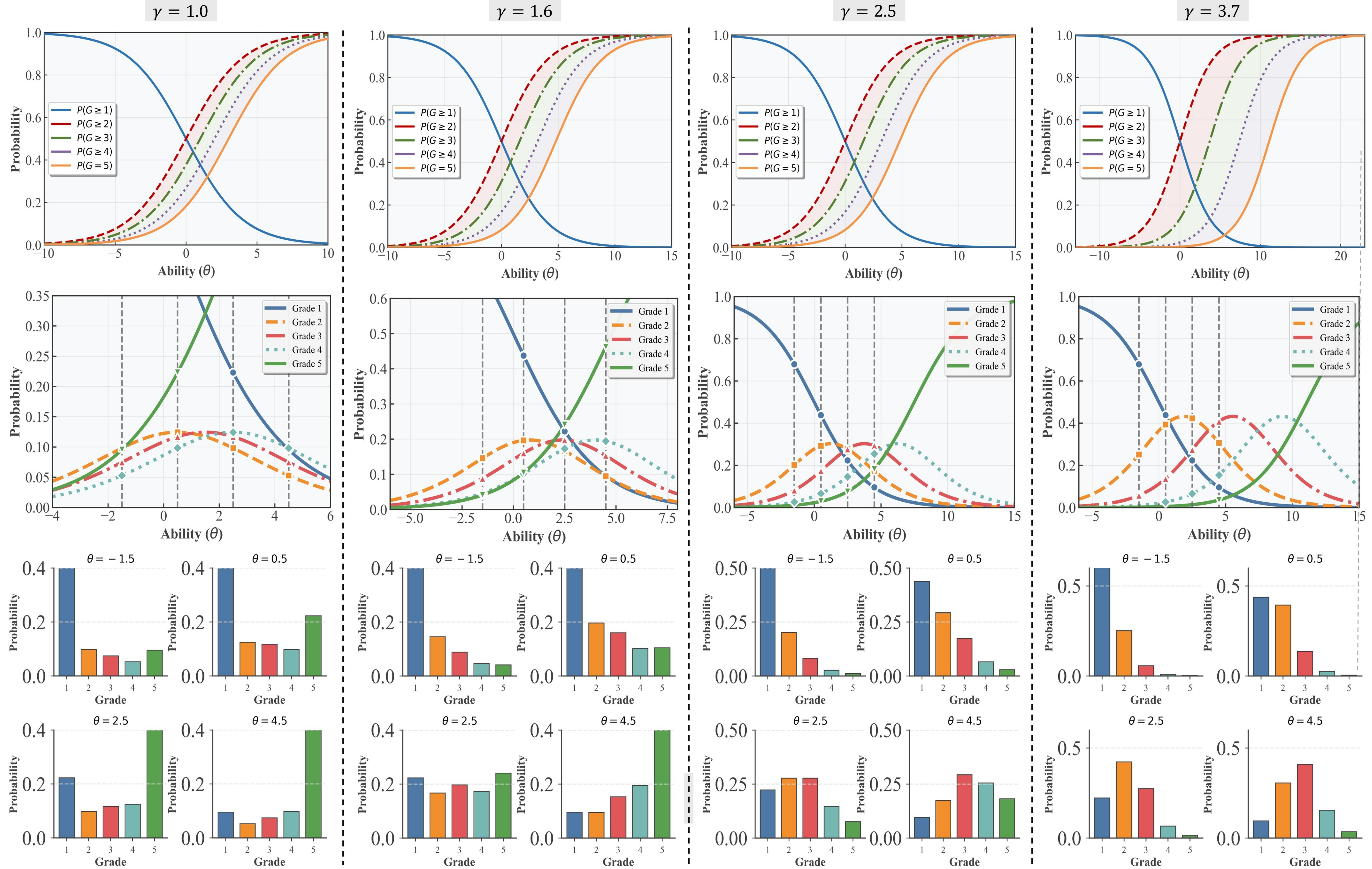}
\caption{ Distributions of the cumulative category probability $P^*(\theta)$ (first row), category response probability $P(\theta)$ (second row), and the category response probability at four specific  $\theta$ (third row). With increasing $\gamma$, the response distributions exhibit enhanced unimodality, aligning more closely with the expected behavior in human quality assessment.}
\label{fig:grm_prob}
\end{figure*}
\subsection{ Preliminaries}
\label{subsec: Preliminaries}
The GRM is widely applied in psychometric research, particularly for assessing self-reported psychological constructs such as anxiety, depression severity, or personal ability. It provides a flexible framework for analyzing polytomous response data, capturing the graded nature of psychological assessments. For example, in an intelligence test, responses to a cognitive reasoning task can be graded on a scale from 0 (``Incorrect") to \( K \) (``Fully correct"), reflecting varying levels of partial correctness. 
For a given item \(j\), the probability that individual \(i\) assigns a rating of at least grade  \(k\) is given by:

\begin{equation}
\label{eqn:1}
    \displaystyle
P_{ij,k}^* = P\left(Y_{ij} \geq k \mid \theta_i\right) \\[1.5ex]
= \frac{1}{1 + e^{-D a_j(\theta_i - \beta_{j,k})}},
\end{equation}
where $P_{i j, k}^*$ is the cumulative category probability at the $k$-th grade and the latent trait \( \theta_i \) represents the cognitive ability of individual $i$. The threshold parameters \( \beta_{j,k} \) define the difficulty levels required to achieve the response category $k$ for item \(j\). \(D\) is a scaling constant that determines the steepness of the item characteristic curve, and \(\alpha\) is the discrimination parameter, indicating how effectively the item differentiates between individuals with different ability levels. Based on Eqn.~(\ref{eqn:1}), the 
probability of an individual grading response category \(k\) can be estimated by:
\begin{equation}
\label{eqn:probs}
P_{i,j,k} = 
\begin{cases}
    \displaystyle
    1 - \frac{1}{1 + e^{-D\alpha(\theta_i - \beta_{j,1})}}, & \text{if } k = 1, \\[1.5ex]
    
    \displaystyle
    \frac{1}{1 + e^{-D\alpha(\theta_i - \beta_{j,k-1})}} - P_{ij,k}^*, & \text{if } 2 \leq k \leq K-1, \\[1.5ex]
    
    \displaystyle
    \frac{1}{1 + e^{-D\alpha(\theta_i - \beta_{j,K-1})}}, & \text{if } k = K.
\end{cases}
\end{equation}
In the following, we omit the subscripts $i$ and $j$ in $P_{i,j,k}$ for simplicity of notation.

\subsection{AGRM-based Quality Grading Module}
\noindent \textbf{Arithmetic Graded Response Model.} Although the GRM is widely used for item response estimation, two limitations remain:

\textbf{(1)} The difficulty levels should be ordinal. Specifically, the threshold parameter of the \((k+1)\)-th level must be greater than the \(k\)-th one, \ie, \( \beta_{k+1} \geq \beta_{k} \) for all \( k \geq 2 \).

\textbf{(2)} The model cannot guarantee the unimodality of the category probability distribution. In other words, it can not preclude the possibility that \( (P_{k^{-}} \geq P_{k}) \land (P_{k^{+}} \geq P_{k}) \) with \( k^{-} < k \) and \( k^{+} > k \), when the target category is the $k$-th item. Such a multimodal probability distribution is unreasonable, as a subject's ability should not simultaneously align more closely with both higher and lower standards compared to the middle-level one.

To account for this, we propose an Arithmetic Graded Response Model (AGRM). Specifically, we define the threshold parameters \( \beta_{1} \) to \( \beta_{K} \) as an arithmetic sequence, \ie,
\begin{align}
    \beta_k &= \beta_1 + (k - 1)\gamma, \quad \text{for } k = 2,3,\dots,K, \\
    &\text{\textit{s.t.}} \quad \gamma > \frac{2\ln 2}{D\alpha}, \quad D > 0, \quad \alpha > 0. \nonumber
\end{align}
This constraint naturally resolves the first limitation. To tackle the second limitation, we demonstrate that the AGRM ensures a unimodal distribution of response probabilities, as shown in Proofs \ref{sec:proofs}. We present several intuitive examples in Fig.~\ref{fig:grm_prob}. 
From the figure, we could observe that:  
(1) The distributions of \( P_1 \) and \( P_K \) are symmetric due to the ``1-" operation.  
(2) For \( k = 2, \dots, K{-}1 \), the distributions are translations of one another, with a translation distance of \( \gamma \). As \( \gamma \) increases, the separation between adjacent distributions becomes larger, making the categorization probabilities more likely to exhibit a unimodal distribution.
When \( \gamma > \frac{2 \ln 2}{D \alpha} > 0 \), the following properties are proven to hold:
\begin{itemize}
    \item \textbf{Maximum Uniqueness.} For any \( \theta \), there exists a unique \( k^* \in [1, K] \) such that \( P_{k^*} \geq P_k \) for all \( k \neq k^* \).  
    \item \textbf{Monotonic Decay.} For any \( \theta \), \( P_k \) decreases monotonically as \( |k - k^*| \) increases when \( P_{k^*} \geq P_k \).  
\end{itemize}
These conditions guarantee the unimodality of the category response probability distribution, mitigating the intrinsic limitations of the traditional GRM model.\\

\begin{figure}[htp]
\centering
\includegraphics[width=0.5\textwidth]{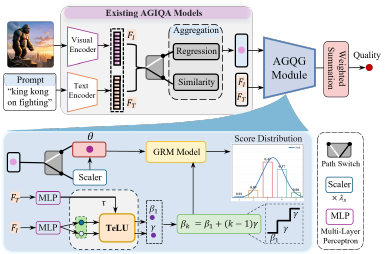}
\caption{Illustration of our AGQG enhancement framework. We adopt the existing AGIQA model for ability estimation and augment it with an additional branch for difficulty estimation. Quality categorization is then performed by the proposed AGRM model using the estimated image ability and quality difficulty. The final quality score is computed as a weighted summation of the categorization results.}
\label{fig:frame}
\end{figure}
\noindent \textbf{AGRM-based Quality Grading Module.} 
The categorization philosophy of AGRM perfectly aligns with human quality-grade ratings. Specifically, image quality can be interpreted as the subject’s ability \( \theta \), while different quality grade standards correspond to different difficulty thresholds \( \beta_{1},\beta_{2},\dots, \beta_{K} \).  Motivated by this insight, we introduce the Arithmetic GRM based Quality Grading (AGQG) module to enhance existing AGIQA models in a plug-and-play manner. Given an image \( I \), a text template \( T \), and their corresponding features \( F_I \) and \( F_T \) extracted by an arbitrary AGIQA model, our AGQG module can be formulated as follows:

\begin{subequations}\label{eq:transfer}
\begin{align}
\theta &= f_{abl}(F_T, F_I), \\
\beta_1 , \gamma &= f_{diff}(F_T, I),
\end{align}
\end{subequations}
where \( f_{abl}(\cdot) \) represents the image ability estimation function, and \( f_{diff}(\cdot) \) predicts both the first quality-grade difficulty \( \beta_1 \) and the difficulty interval \( \gamma \) between consecutive grades. Based on the predicted \( \beta_1 \) and \( \gamma \),  difficulty threshold for the \( k \)-th quality grade (\( k > 1 \)) is determined as:

\begin{equation}
\beta_k = \beta_1 + (k - 1)\gamma, \quad \text{for } k = 2,3,\dots,K.
\end{equation}
As illustrated in Fig.~\ref{fig:frame}, the proposed AGQG module comprises two branches dedicated to image ability estimation and quality difficulty estimation, respectively.

\quad

\noindent (1) \textbf{Ability branch.} 
The ability branch preserves the original text-image feature aggregation function \(f_{\text{agg}}(\cdot)\) from AGIQA, while introducing a rescaling layer to regulate its output. When the final stage of \(f_{\text{agg}}\) applies a softmax activation, the ability score is computed by scaling the output with a fixed sensitivity hyperparameter \(\lambda_s\), enabling the model to represent a broader range of ability parameters. Otherwise, the output of \(f_{\text{agg}}\) is used directly. Formally,
\begin{equation}
\label{eq:theta}
\theta =
\begin{cases}
\lambda_s \cdot f_{\text{agg}}(F_T, F_I), & \text{if softmax is applied}, \\[6pt]
f_{\text{agg}}(F_T, F_I), & \text{otherwise}.
\end{cases}
\end{equation}

\quad

\noindent (2) \textbf{Difficulty branch.} 
The difficulty branch aims to estimate two parameters, \(\beta_1\) and \(\gamma\), which describe the quality difficulty of the given image. Herein, we design the estimation process based on the observation that the text features \(F_T\) inherently contain high-level semantic cues, which provide a reliable basis for estimating the base difficulty parameters \(\beta_1^T\) and \(\gamma^T\). These text-derived parameters serve as coarse priors and are subsequently modulated by the image-specific features \(F_I\) to capture instance-dependent visual variations. Specially, we first project the text features into two initial difficulty parameters :
\begin{equation}
\beta_1^T = \phi_{\beta}(F_T), 
\quad
\gamma^T = \phi_{\gamma}(F_T),
\label{enphi}
\end{equation}
where \(\phi_{\beta}(\cdot)\) and \(\phi_{\gamma}(\cdot)\) are separate learnable mappings.
Then the image features are processed  to produce the temperature coefficient:
\begin{equation}
\tau_I = \phi_I(F_I).
\end{equation}
The final difficulty parameters are computed by adaptively modulating the text-derived priors with the image-dependent temperature coefficient, based on the Temperatured TeLU activation function~\cite{fernandez2024telu}:
\begin{align}
\beta_1 &= (\beta_1^T + \tau_I) \cdot \tanh\!\big(\exp(\beta_1^T + \tau_I)\big), \\
\gamma   &= (\gamma^T + \tau_I) \cdot \tanh\!\big(\exp(\gamma^T + \tau_I)\big) + \eta,
\end{align}
where \(\eta\) is a small positive hyperparameter ensuring positivity and minimal separation between \(\beta_1\) and \(\gamma\). This formulation enables the difficulty branch to combine stable semantic priors with adaptive modulation, allowing the estimated difficulty to reflect both task-related semantics and image-specific complexity.

Based upon the obtained \(\theta\), \(\beta_1\), and \(\gamma\), the probability of an image being categorized into the \(k\)-th quality grade can be estimated by:
\begin{equation}
\label{eqn:pkk}
P_k = \begin{cases}
    1 - \dfrac{1}{1 + e^{-D\alpha(\theta - \beta_{1})}}, & \text{if } k = 1, \\[4pt]
    \dfrac{\varphi_k(e^{D\alpha\gamma} - 1)}{(1 + \varphi_k)(1 + \varphi_k \cdot e^{D\alpha\gamma})}, & \text{if } 2 \leq k \leq K-1, \\[4pt]
    \dfrac{1}{1 + e^{-D\alpha(\theta - \beta_{1} - (K-2)\gamma)}}, & \text{if } k = K,
\end{cases}
\end{equation}
with
\begin{equation}
\varphi_k = e^{-D\alpha(\theta - \beta_1 - (k - 2)\gamma)},
\end{equation}
where \(D\) and \(\alpha\) are fixed positive parameters shared across different AGIQA models.

\subsection{Plug-and--Play Integration into AGIQA Models}
As shown in Fig.~\ref{fig:frame},  our AGRM-QC module can be seamlessly integrated into various AGIQA models in a plug-and-play manner, as it operates independently of the framework design of the image encoder $f_I(\cdot)$ and text encoders $f_T(\cdot)$.  In the integration process, different AGIQA models usually employ distinct text templates tailored to the multi-dimensional quality assessment. Herein, we omit their predefined quality-grade descriptions to describe the image’s content and quality and \textbf{avoid linguistic vagueness}. The unifying text template $T$  is defined as follows:

\noindent  $\bullet$ For the perceptual quality assessment,  $T$ = ``A photo of good quality and clear details''. \\
\noindent  $\bullet$ For the image authenticity assessment,  $T$ = ``'A photo with genuine scene content and no synthetic artifacts''. \\
\noindent  $\bullet$ For the prompt alignment assessment,  we set $T$ as the image generation prompt by default.

Once the quality categorization result \( P_{k} \) is obtained from Eqn.~(\ref{eqn:pkk}), the image quality score \( Q \) (which can represent perceptual quality, authenticity quality, or alignment quality) is estimated via a weighted aggregation:
\begin{equation}
Q = \sum_{k=1}^{K} k \cdot P_{k}, \quad k = 1, 2, \ldots, K.
\end{equation}
During training, we employ a combination of mean absolute error (MAE) loss \( \mathcal{L}_{{MAE}} \) and Pearson linear correlation coefficient (PLCC) loss \( \mathcal{L}_{{PLCC}} \), which can be formulated by:
\begin{equation}
\label{tot_loss}
\mathcal{L} = \mathcal{L}_{{MAE}} + \lambda\mathcal{L}_{{PLCC}},
\end{equation}
where \( \lambda \) are hyperparameters that balance the contributions of different loss terms. The MAE loss is defined as:
\begin{align}
\mathcal{L}_{{MAE}} &= \frac{1}{N}\sum_{i=1} ^{N}\big| Q^T_i - Q_i\big|,\\
 Q_i&= (\tilde{Q}_i - 1) \times \frac{5}{K-1} ,
\end{align}
where $N$ is the bach size, \(Q^T_i \) and \(\tilde{Q}_i \) are the ground truth and predicted quality scores of the $i$-th image, respectively. The PLCC loss is defined by:
\begin{align}
\mathcal{L}_{\text{PLCC}} &=  \frac{1}{N}\left(\lVert \hat{Q_i} - \hat{Q}^T_i \rVert_2^2 + \lVert \rho \cdot \hat{Q} - {Q}^T \rVert_2^2\right) ,  
\end{align}
where
\begin{align}
\hat{Q}_i &= \frac{Q_i - \mu_Q}{\sigma_Q + \epsilon}, \quad
\hat{Q}^T_i = \frac{Q^T_i - \mu_{Q^T}}{\sigma_{Q^T} + \epsilon}, \quad
\rho = \frac{1}{N} \sum_{i=1}^N \hat{Q}_i \cdot \hat{Q}_{i}^T, 
\end{align}
with $\mu_Q$ and $\sigma_{Q}$ being the mean and standard deviation of the predicted quality scores in a batch, and $\mu_{Q^T}$ and $\sigma_{Q^T}$ being the corresponding ground truth values.

\section{Experiments}

\subsection{Experimental Settings}
\textbf{Datasets.}
We conduct experiments on three benchmark datasets: AGIQA-1K\cite{zhang2023perceptual}, AGIQA-3K\cite{li2023agiqa}, and AIGCIQA2023\cite{wang2023aigciqa2023}.  
The AGIQA-1K dataset comprises 1,080 AGIs synthesized using two latent text-to-image diffusion models: Stable Inpainting
V1 and Stable Diffusion V2\cite{rombach2022high}. Constructed from diverse textual prompts, it covers a broad range of semantic content, including birds, humans, urban, and natural landscapes. The dataset spans multiple artistic styles, from anime to photorealistic renderings, ensuring variation in both structure and aesthetics. Each image is annotated with a Mean Opinion Score (MOS) of human perceptual quality.
The AGIQA-3K dataset comprises 2,982 images generated by six state-of-the-art models: AttnGAN \cite{gen:AttnGAN}, DALLE2 \cite{gen:DALLE2}, GLIDE\cite{gen:GLIDE}, Midjourney\cite{gen:MJ}, Stable Diffusion\cite{gen:SD} and Stable Diffusion XL\cite{gen:XL}. These models involve a diverse range of generative architectures, including GAN-based, auto-regressive, and diffusion-based approaches. Each image is annotated with MOSs for both perceptual quality and prompt alignment, facilitating a comprehensive assessment of text-to-image consistency. 
The AIGCIQA2023 dataset consists of 2,400 images generated by six generative models. These images are derived from prompts spanning 100 distinct scenarios, and each scenario generates four possible images. The assessment dimensions focus on three key perspectives: perceptual quality, image authenticity, and prompt consistency (alignment).\\

\begin{table*}[!t]
\centering
\renewcommand{\arraystretch}{1.1}
\caption{Performance comparison under the intra-dataset setting on the AGIQA-3K\cite{li2023agiqa} and AIGCIQA2023\cite{wang2023aigciqa2023}  datasets. ``AGQG +" denotes incorporation of our AGQG  module. The best and second-best results are highlighted in \textbf{bold} and \underline{underline}, respectively.}
\vspace{-10pt}
\label{tab:comparison}
\resizebox{1.0\textwidth}{!}{
\begin{tabular}{l|c c|c c| c c|c c|c c}
    \shline
    \rowcolor[HTML]{EFEFEF}
    \multicolumn{1}{l|}{\cellcolor[HTML]{EFEFEF}} &\multicolumn{4}{c|}{\cellcolor[HTML]{EFEFEF}AGIQA-3K} &
      \multicolumn{6}{c}{\cellcolor[HTML]{EFEFEF}AIGCIQA2023} \\
    \hhline{>{\arrayrulecolor[HTML]{EFEFEF}}->{\arrayrulecolor{black}}|----------}
    \rowcolor[HTML]{EFEFEF}
    \multicolumn{1}{l|}{\cellcolor[HTML]{EFEFEF}} &
      \multicolumn{2}{c|}{Quality} &\multicolumn{2}{c|}{Consistency} &\multicolumn{2}{c|}{Quality} &\multicolumn{2}{c|}{Consistency} &
      \multicolumn{2}{c}{Authenticity} \\
     \hhline{>{\arrayrulecolor[HTML]{EFEFEF}}->{\arrayrulecolor{black}}|----------}
    \rowcolor[HTML]{EFEFEF}
    \multicolumn{1}{l|}{\multirow{-3}{*}{\cellcolor[HTML]{EFEFEF}\textbf{Method}}}& SRCC & PLCC & SRCC & PLCC & SRCC & PLCC & SRCC & PLCC & SRCC & PLCC \\
    \hline \hline
    VGG16\cite{simonyan2014very} & 0.8167 & 0.8752 & 0.6867 & 0.8108 & 0.8157 & 0.8282 & 0.6839 & 0.6853 & 0.7399 & 0.7465  \\
    ResNet50\cite{he2016deep} & 0.8552 & 0.9072 & 0.7493 & 0.8564 & 0.8190 & 0.8503 & 0.7230 & 0.7270 & 0.7571 & 0.7563 \\
    ViT/B/16\cite{dosovitskiy2020image} & 0.8659 & 0.9115 & 0.7478 & 0.8449 & 0.8370 & 0.8618 & 0.7293 & 0.7439 & 0.7783 & 0.7697  \\
    ViL\cite{alkin2024visionlstm} & 0.8750 & 0.9145 & 0.7570 & 0.8561 & 0.8436 & 0.8770 & 0.7174 & 0.7296 & 0.7753 & 0.7770  \\ 
    \hline
    DBCNN\cite{zhang2018blind} & 0.8154 & 0.8747 & 0.6329 & 0.7823 & 0.8339 & 0.8577 & 0.6837 & 0.6787 & 0.7485 & 0.7436 \\
    StairIQA\cite{sun2022blind} & 0.8439 & 0.8989 & 0.7345 & 0.8474 & 0.8264 & 0.8483 & 0.7176 & 0.7133 & 0.7596 & 0.7514  \\
    MGQA\cite{wang2021multi} & 0.8283 & 0.8944  & 0.7244 & 0.8430 & 0.8093 & 0.8308 & 0.6892 & 0.6745 & 0.7367 & 0.7310  \\
    HyperIQA\cite{Su_2020_CVPR} & 0.8526 & 0.8975 & 0.7437 & 0.8471 & 0.8357 & 0.8504 & 0.7318 & 0.7222 & 0.7758 & 0.7790   \\
    AMFF-Net\cite{zhou2024adaptive} & 0.8565 & 0.9050 & 0.7513 & 0.8476 & 0.8409 & 0.8537 & 0.7782 & 0.7638 & 0.7749 & 0.7643  \\
    SSL-IPF\cite{zhao2025self} & 0.8523 & 0.9045  & - & - & 0.8632 & 0.8678 & 0.7958 & 0.7911 & 0.7774 & 0.7617  \\
    \hline
    TIER \cite{yuan2024tier} & 0.8251 & 0.8821 & 0.6541 & 0.7968 & 0.8194 & 0.8359 & 0.7033 & 0.6966 & 0.7323 & 0.7226  \\
    CLIP-IQA \cite{wang2023exploring} & 0.8426 & 0.8053 & 0.6720 & 0.8043 & 0.7802 & 0.8140 & 0.7040 & 0.6941 & 0.6726 & 0.6627  \\
    CLIP-AGIQA \cite{tang2025clip} & 0.8618 & 0.8978 & - & - & 0.8140 & 0.8302 & - &  - & 0.7940 & 0.7797 \\
    IPCE\cite{peng2024aigc} & 0.8841 & 0.9246 & \underline{0.7697} & \underline{0.8725} & \underline{0.8640} & 0.8788 & \underline{0.7979} & \underline{0.7887} & 0.8097 & 0.7998 \\
\hline
\rowcolor{mygray}
    AGQG +TIER & 0.8421 & 0.8977 & 0.6663 & 0.8054 & 0.8376 & 0.8586 & 0.7445 & 0.7360 & 0.7676 & 0.7586 \\
    \rowcolor[HTML]{FFEEED}
    \textit{Improvement}
  & $\uparrow\best{2.06\%}$ & $\uparrow\best{1.77\%}$ & $\uparrow\best{1.87\%}$ & $\uparrow\best{1.08\%}$ & $\uparrow\best{2.22\%}$ & $\uparrow\best{2.72\%}$ & $\uparrow\best{5.86\%}$ & $\uparrow\best{5.66\%}$ & $\uparrow\best{4.82\%}$ & $\uparrow\best{4.98\%}$ \\
    \rowcolor{mygray}
    AGQG +CLIP-IQA  & 0.8614 & 0.9106 & 0.7467 & 0.8546 & 0.8493 & 0.8685 & 0.7885 & 0.7806 & 0.7919 & 0.7825  \\
    \rowcolor[HTML]{FFEEED}
    \textit{Improvement}
  & $\uparrow\best{2.23\%}$ & $\uparrow\best{13.08\%}$ & $\uparrow\best{11.12\%}$ & $\uparrow\best{6.25\%}$ & $\uparrow\best{8.86\%}$ & $\uparrow\best{6.70\%}$ & $\uparrow\best{12.00\%}$ & $\uparrow\best{12.46\%}$ & $\uparrow\best{17.73\%}$ & $\uparrow\best{18.08\%}$ \\
    \rowcolor{mygray}
    AGQG +CLIP-AGIQA & \underline{0.8864} & \underline{0.9268} & - & - & 0.8600 & \underline{0.8804} & - & - & \textbf{0.8131} & \textbf{0.8036} \\
    \rowcolor[HTML]{FFEEED}
    \textit{Improvement}
  & $\uparrow\best{2.85\%}$ & $\uparrow\best{3.23\%}$ & – & – & $\uparrow\best{5.65\%}$ & $\uparrow\best{6.05\%}$ & – & – & $\uparrow\best{2.41\%}$ & $\uparrow\best{3.07\%}$ \\
    \rowcolor{mygray}
    AGQG +IPCE  & \textbf{0.8911} & \textbf{0.9284} & \textbf{0.7717} & \textbf{0.8741} & \textbf{0.8671} & \textbf{0.8846} & \textbf{0.8052} & \textbf{0.7978} & \underline{0.8115} & \underline{0.8031} \\
    \rowcolor[HTML]{FFEEED}
    \textit{Improvement}
  & $\uparrow\best{0.79\%}$ & $\uparrow\best{0.41\%}$ & $\uparrow\best{0.26\%}$ & $\uparrow\best{0.18\%}$ & $\uparrow\best{0.36\%}$ & $\uparrow\best{0.66\%}$ & $\uparrow\best{0.91\%}$ & $\uparrow\best{1.15\%}$ & $\uparrow\best{0.22\%}$ & $\uparrow\best{0.41\%}$ \\
    \shline
\end{tabular}}
\end{table*}

\noindent\textbf{Implementation Details.}
To verify the effectiveness and generalization capability of our AGQG   module, we incorporate it into four existing AGIQA models, including CLIP-IQA\cite{wang2023exploring}, TIER\cite{yuan2024tier}, CLIP-AGIQA\cite{tang2025clip}, and IPCE\cite{peng2024aigc}. Specifically, we set \(\lambda_s\) in Eqn.~(\ref{eq:theta}) to 10.0 and \(\lambda\) in Eqn.~(\ref{tot_loss}) to 1.0. The parameters \(D\) and \(\alpha\) in Eqn.~(\ref{eqn:pkk}) are set to 1.7 and 1.0, respectively.
We employ the AdamW optimizer\cite{loshchilov2017decoupled} to train the models, with a weight decay of \(10^{-3}\) and learning rates of \(1 \times 10^{-5}\). The learning rate is adjusted using a cosine annealing scheduler\cite{10104453} with a maximum period of 5 epochs. Each model is trained for 100 epochs with a batch size of 16. 
All experiments are conducted on a single NVIDIA GeForce RTX 3090 GPU. For performance assessment, we adopt the Spearman Rank-Order Correlation Coefficient (SRCC) and Pearson's Linear Correlation Coefficient (PLCC) as comparison metrics. Specifically, the SRCC measures the monotonic relationship between the predicted results and MOS, while PLCC quantifies the linear correlation between the two variables. Higher values of both SRCC and PLCC indicate better model performance.

\subsection{Performance Enhancement on  Intra-dataset}
\label{subsec:transferability}
To validate the effectiveness of our AGQG  module, we integrate it into four existing AGIQA frameworks, which can be categorized into two groups: 1) Quality regression based model: TIER, where quality is regressed from a concatenated representation of features extracted from the images and their corresponding prompts. 2) Quality categorization based models: CLIP-IQA, CLIP-AGIQA, and IPCE, where the image quality is graded into a set of quality-level descriptions, and the final quality score is obtained by an aggregation of the grading probabilities. The experiments are conducted on AGIQA-3K, AIGCIQA2023, and AGIQA-1K datasets, with results presented in Table~\ref{tab:comparison} and Table~\ref{tab:agiqa1k-p}. From these tables, we could observe: 

\quad

\noindent \textbf{(1) Consistent Performance Gains Across Frameworks and Datasets}: The integration of our AGQG  module consistently improves the performance across various AGIQA frameworks, datasets, and evaluation dimensions, demonstrating the strong generalization capability of our approach. In Fig.~\ref{fig:scatter}, we further present the scatter plots of our predicted scores against the MOSs. The higher linearity and the enhanced proximity of the scatter points to the ideal identity line demonstrate that the integration of our AGQG  module yields predictions that are significantly more consistent with human perceptual judgments.
To further evaluate the quality discrimination ability in separating high- and low-quality images, we visualize the density distributions of predicted quality scores in Fig.~\ref{fig:density}. 
Specifically, high- and low-quality samples are separated using the mean of MOSs as a threshold. Across all datasets, the AGQG integrated models exhibit consistently lower overlap between the two distributions compared to their original counterparts. 
This reduced overlap demonstrates the enhanced capability of AGQG  to capture intrinsic quality characteristics and better distinguish between different quality levels.

\noindent \textbf{(2) Addressing Semantic Drift in Quality Grading Models}:  
As illustrated  in Fig.~\ref{fig:prob}, existing quality-grading models usually produce multi-peak score distributions, indicating misalignment 
between the learned text embeddings and their intended semantic descriptions of quality levels.
Notably, this issue is observed even in the state-of-the-art model, IPCE.
By integrating the proposed  AGQG  module, the distributions become consistently unimodal across different images, effectively mitigating semantic drift and leading to improved assessment performance.


\noindent \textbf{(3) Reducing Dependence on Prompt Optimization}:
The performance of CLIP-based IQA models is highly sensitive to prompt design. For example, CLIP-AGIQA, which incorporates learnable prompts, achieves noticeably better results than the original CLIP-IQA with fixed prompts, highlighting the critical role of prompt optimization in quality prediction. In contrast, our AGQG  module mitigates the reliance on prompt optimization by utilizing fixed, task-specific textual templates. Despite using non-learnable prompts, our method still surpasses CLIP-IQA by over 10\% in average PLCC across all  assessment dimensions, demonstrating its ability to achieve stable and accurate predictions while greatly simplifying the training process.

\begin{table}[!t]
\centering
\renewcommand{\arraystretch}{1.2}
\caption{Performance comparison results under the intra-dataset setting on the  AGIQA-1K\cite{zhang2023perceptual} dataset. ``AGQG +" denotes incorporation of our AGQG  module. The best and second-best results are highlighted in \textbf{bold} and \underline{underline}, respectively.}
\vspace{-10pt}
\label{tab:agiqa1k-p}
\resizebox{\columnwidth}{!}{
\begin{tabular}{l|c|ccc}
\shline
\rowcolor[HTML]{EFEFEF} 
\textbf{Methods} & \textbf{Year} & \textbf{SRCC} & \textbf{PLCC} & \textbf{Average} \\ \hline \hline
CNNIQA\cite{kang2014convolutional} & \textit{2014} & 0.5800 & 0.7139 & 0.6470 \\
DBCNN\cite{zhang2018blind} & \textit{2018} & 0.7491 & 0.8211 & 0.7851 \\
MGQA\cite{wang2021multi} & \textit{2021} & 0.6011 & 0.6760 & 0.6386 \\
StairIQA\cite{sun2023blind} & \textit{2023} & 0.5504 & 0.6088 & 0.5796 \\
PSCR\cite{yuan2023pscr} & \textit{2023} & 0.8430 & 0.8403 & 0.8417 \\
\midrule
TIER\cite{yuan2024tier} & \textit{2024} & 0.8266 & 0.8297 & 0.8282 \\
CLIPIQA\cite{wang2023exploring} & \textit{2023} & 0.8227 & 0.8411 & 0.8319 \\
CLIP-AGIQA\cite{tang2025clip} & \textit{2024} & 0.7918 & 0.8339 & 0.8129 \\
IPCE\cite{peng2024aigc} & \textit{2024} & {0.8535} & 0.8792 & 0.8664 \\
\hline
\rowcolor{mygray}
AGQG +TIER & - & 0.8312 & 0.8679 & 0.8496 \\
\rowcolor[HTML]{FFEEED}
\textit{Improvement} & - & $\uparrow$ \best{0.56\%} & $\uparrow$ \best{4.60\%} & $\uparrow$ \best{2.58\%} \\
\rowcolor{mygray}
AGQG +CLIP-IQA & - & 0.8402 & 0.8707 & 0.8555 \\
\rowcolor[HTML]{FFEEED}
\textit{Improvement} & - & $\uparrow\best{2.13\%}$ & $\uparrow$ \best{3.52\%} & $\uparrow$ \best{2.84\%} \\
\rowcolor{mygray}
AGQG +CLIP-AGIQA & - & \underline{0.8544} & \underline{0.8831} & \underline{0.8688} \\
\rowcolor[HTML]{FFEEED}
\textit{Improvement} & - & $\uparrow$ \best{7.91\%} & $\uparrow$ \best{5.90\%} & $\uparrow$ \best{6.88\%} \\
\rowcolor{mygray}
AGQG +IPCE & - & \textbf{0.8606} & \textbf{0.8881} & \textbf{0.8744} \\
\rowcolor[HTML]{FFEEED}
\textit{Improvement} & - & $\uparrow$ \best{0.83\%} & $\uparrow$ \best{1.01\%} & $\uparrow$ \best{0.92\%} \\
\shline
\end{tabular}}
\end{table}

\begin{figure*}[!t]
\centering
\includegraphics[width=\textwidth]{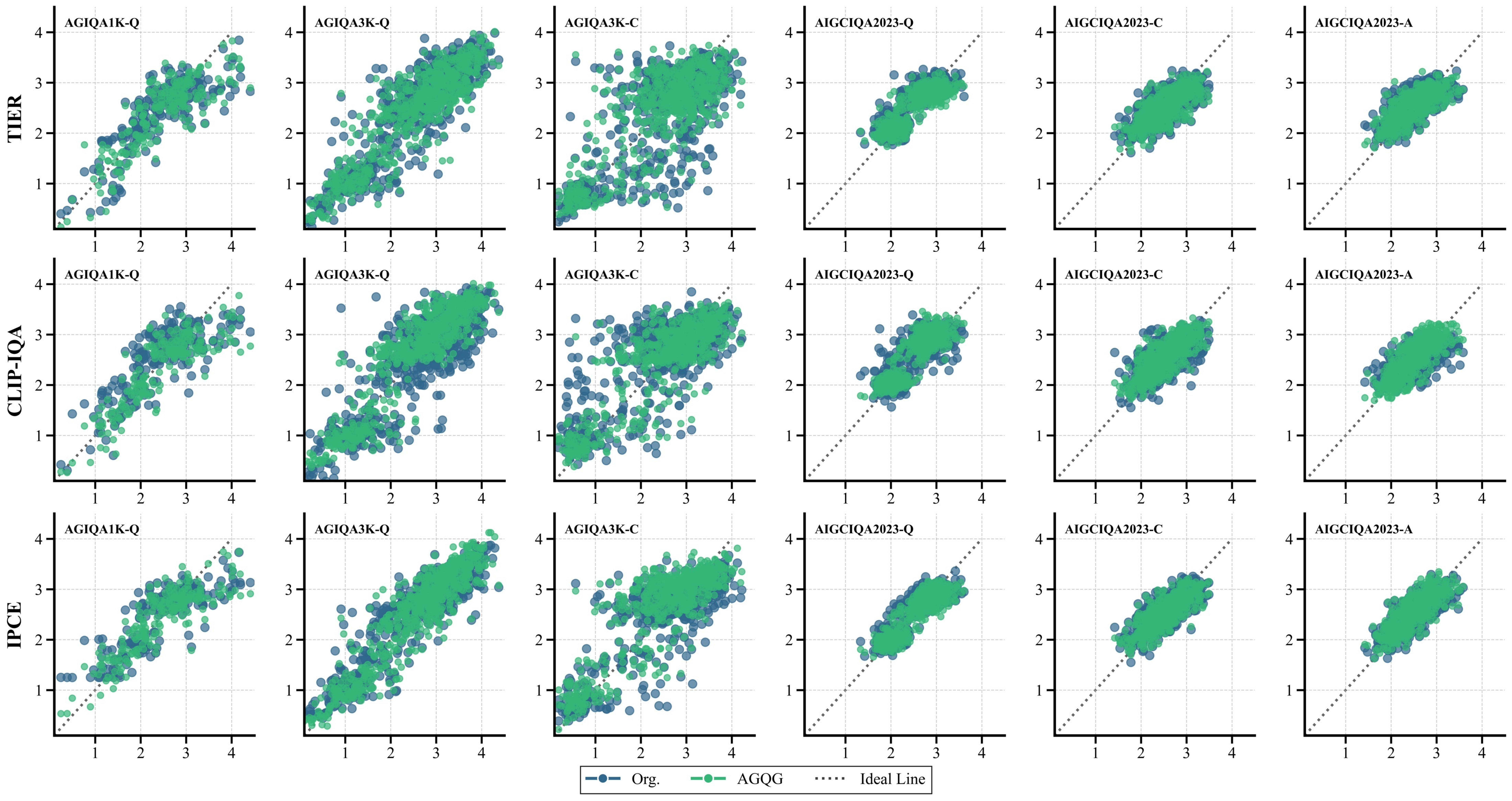}
\vspace{-15pt}
\caption{ Scatter plots of three IQA methods including CLIP-IQA, TIER and IPCE on AGIQA-1K\cite{zhang2023perceptual}, AGIQA-3K\cite{li2023agiqa}, and AIGCIQA2023\cite{wang2023aigciqa2023} dataset. The x-axis represents the ground truth MOS, while the y-axis shows the predicted MOS. The blue scatters represent the results of original, and the green scatters represent the results after the integration of AGRM-QE. Q stands for perception quality, C stands for sematic consistency, and A stands for image authenticity. As the scatter gets closer to the ideal line, it indicates that the model predicts better.}
\label{fig:scatter}
\end{figure*}

\begin{figure}[!t]
\centering
\includegraphics[width=0.5\textwidth]{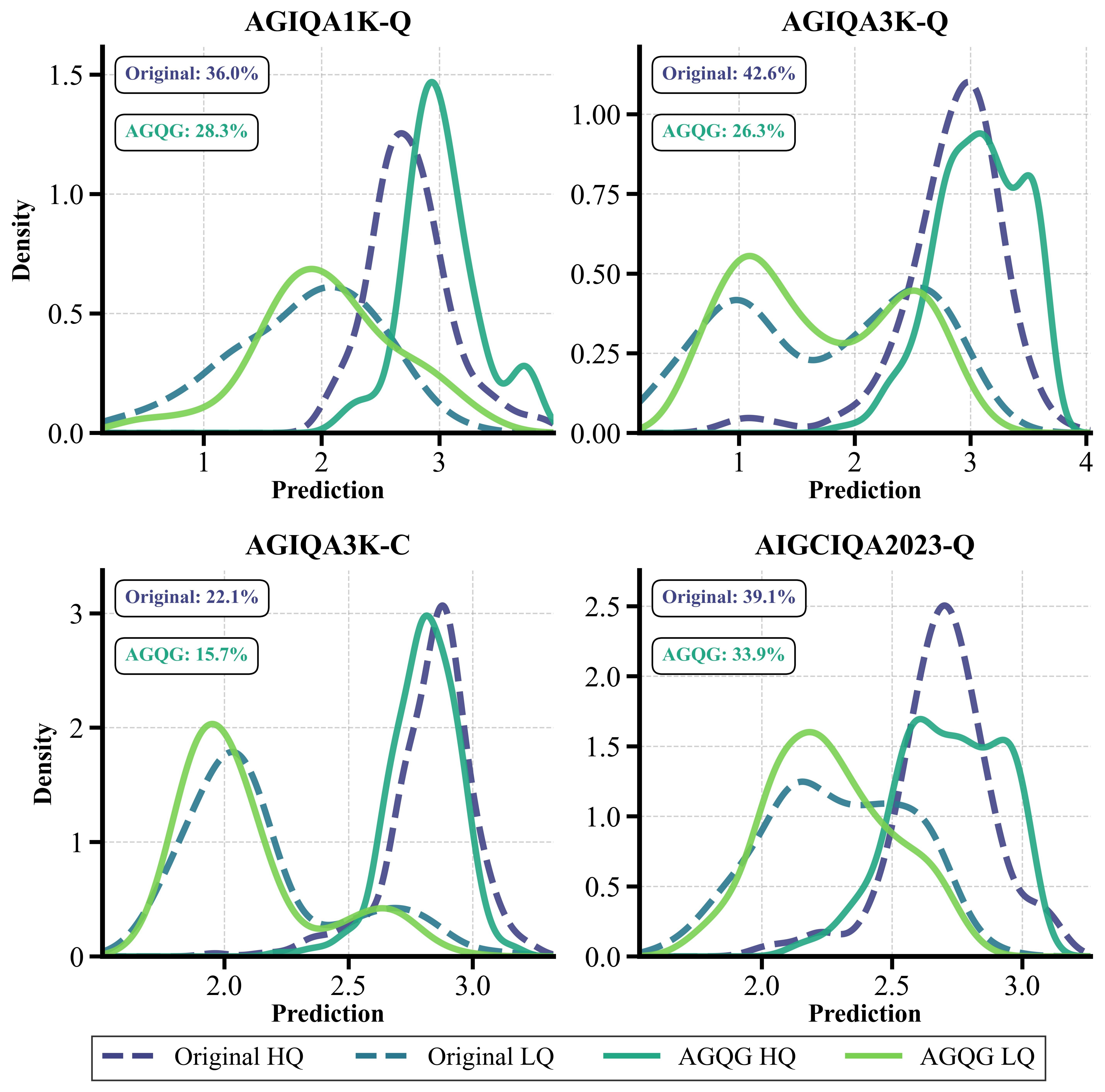}
\vspace{-15pt}
\caption{Density distributions of predicted IQA scores from the original TIER model and its integration with AGQG on AGIQA-1K~\cite{zhang2023perceptual}, AGIQA-3K~\cite{li2023agiqa}, and AIGCIQA2023~\cite{wang2023aigciqa2023} datasets. 
The overlap percentage between high- and low-quality distributions(denoted as HQ and LQ respectively) is annotated in the top-left corner of each subfigure, where lower values indicate better quality discrimination.
}
\label{fig:density}
\end{figure}

\begin{figure*}[!t]
\centering
\includegraphics[width=\textwidth]{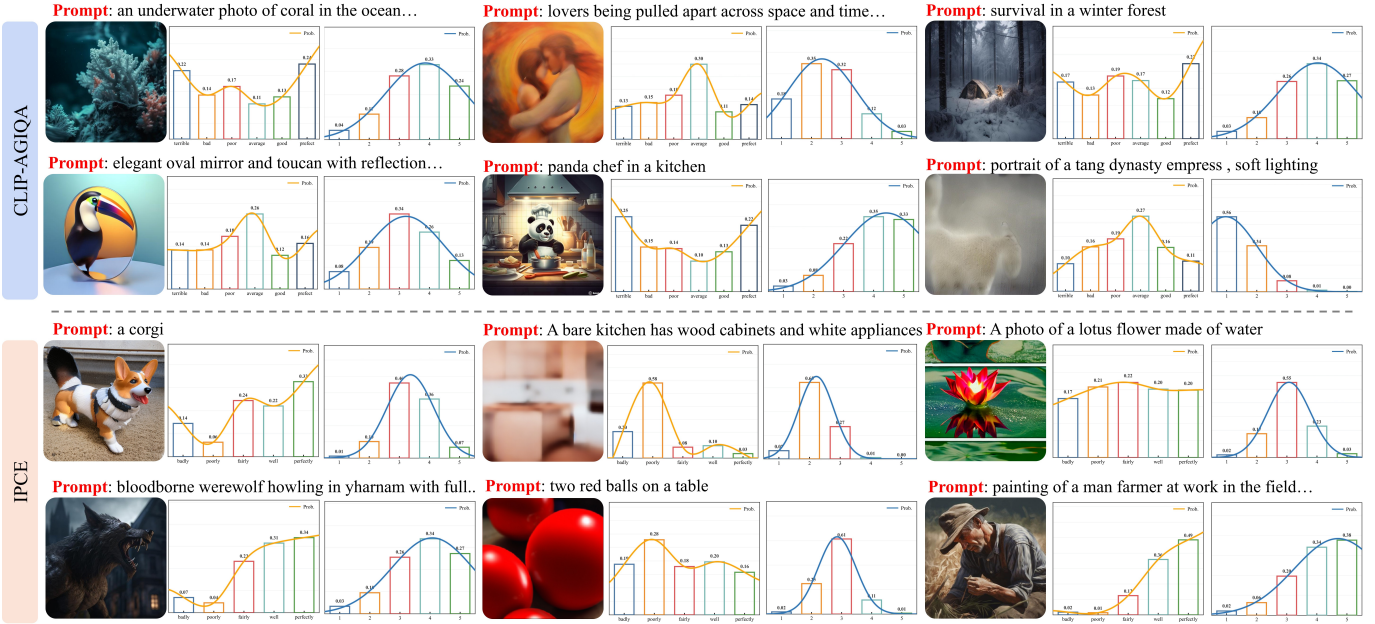}
\vspace{-20pt}
\caption{Visualization of the probability distributions of quality categorization from the original AGIQA models (columns 2, 5, 8) against models enhanced by our AGQG  module (columns 3, 6, 9). Zoomed-in for better details.}
\label{fig:prob}
\end{figure*}

\begin{figure}[!t]
\centering
\includegraphics[width=0.5\textwidth]{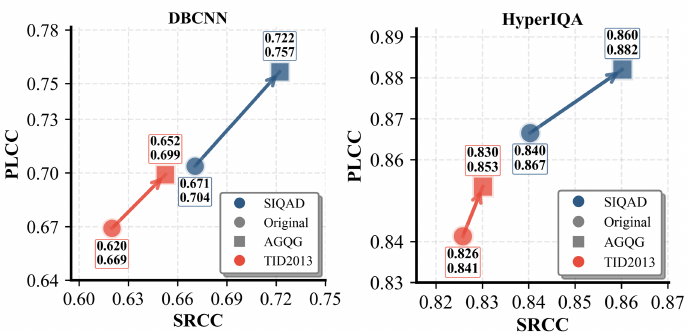}
\vspace{-15pt}
\caption{Impact of AGQG integration on different IQA frameworks, evaluated on the TID2013\cite{ponomarenko2015image} and SIQAD\cite{yang2015perceptual} datasets. Consistent improvements are observed for both DBCNN (left) and HyperIQA (right), with points closer to the top-right corner indicating superior overall performance.}

\label{fig:trad_iqa}
\end{figure}
\subsection{Cross-Dataset Generalization Enhancement}
To further validate the effectiveness of our AGQG  module, we perform cross-dataset evaluations to assess its impact on two cross-dataset settings: 1) training on AIGCIQA2023 and testing on AGIQA-3K (denoted as ``AIGCIQA2023 $\rightarrow$ AGIQA-3K''), and 2) training on AGIQA-3K and testing on AIGCIQA2023 (denoted as ``AGIQA-3K $\rightarrow$ AIGCIQA2023''). The AGQG  module is incorporated into four representative AGIQA models: TIER,CLIP-IQA, CLIP-AGIQA, and IPCE, with performance comparisons summarized in Table~\ref{tab:cross-dataset}.

As shown in the table, all baseline models exhibit notable performance degradation under cross-dataset scenarios, highlighting the challenge of achieving strong generalization. However, consistent performance improvements can be observed across all models when our AGQG  module is incorporated. For instance, on the CLIP-IQA model, AGQG  yields over \textbf{10\%} improvement in both SRCC and PLCC. This gain mainly lies in the adaptive quality standards learned for each image, in contrast to the fixed antonym prompts in the original model, which limit expressiveness for fine-grained quality discrimination. Although IPCE achieves the best performance in both intra- and cross-dataset evaluations, it still suffers a noticeable drop in alignment assessment tasks due to unshared image content and generation artifacts between AIGCIQA2023 and AGIQA-3K datasets. Nevertheless, integrating AGQG  consistently enhances performance, indicating more reliable text-image shared representation space can be learned by our module.

\begin{table*}[htbp]
\centering
\large
\renewcommand{\arraystretch}{1.1}
\caption{Performance comparison results under the cross-dataset setting. ``AGQG +" denotes incorporation of our AGQG  module. The best and second-best results are highlighted in \textbf{bold} and \underline{underline}, respectively.}
\vspace{-10pt}
\label{tab:cross-dataset}
\resizebox{1.0\textwidth}{!}{
\begin{tabular}{l|ccc|ccc|ccc|ccc|c}
\shline
\rowcolor[HTML]{EFEFEF}
\multicolumn{1}{l|}{\cellcolor[HTML]{EFEFEF}} & \multicolumn{6}{c|}{\cellcolor[HTML]{EFEFEF}AIGCIQA2023(Train) $\rightarrow$ AGIQA-3K(Test)} & \multicolumn{6}{c|}{\cellcolor[HTML]{EFEFEF}AGIQA-3K(Train) $\rightarrow$ AIGCIQA2023(Test)} & \multicolumn{1}{c}{\cellcolor[HTML]{EFEFEF}} \\
\hhline{>{\arrayrulecolor[HTML]{EFEFEF}}->{\arrayrulecolor{black}}|------------}
\rowcolor[HTML]{EFEFEF}
\multicolumn{1}{l|}{\cellcolor[HTML]{EFEFEF}} & \multicolumn{3}{c|}{Quality} & \multicolumn{3}{c|}{Consistency} & \multicolumn{3}{c|}{Quality} & \multicolumn{3}{c|}{Consistency} & Avg \\
\hhline{>{\arrayrulecolor[HTML]{EFEFEF}}->{\arrayrulecolor{black}}|------------}
\rowcolor[HTML]{EFEFEF}
\multicolumn{1}{l|}{\multirow{-3}{*}{\cellcolor[HTML]{EFEFEF}\textbf{Methods}}}& SRCC & PLCC & Mean & SRCC & PLCC & Mean & SRCC & PLCC & Mean & SRCC & PLCC & Mean & \\
 \hline \hline
    ResNet50\cite{he2016deep} & 0.576 & 0.612 & 0.594 & 0.473 & 0.523 & 0.498 & 0.599 & 0.609 & 0.604 & 0.432 & 0.435 & 0.433 & 0.533 \\
    ViT-B/32\cite{dosovitskiy2020image} & 0.509 & 0.571 & 0.540 & 0.434 & 0.496 & 0.465 & 0.517 & 0.529 & 0.523 & 0.381 & 0.390 & 0.385 & 0.484 \\
    MUSIQ\cite{ke2021musiq} & 0.635 & 0.673 & 0.654 & 0.394 & 0.437 & 0.416 & 0.650 & 0.643 & 0.646 & 0.525 & 0.515 & 0.520 & 0.576 \\
    DB-CNN\cite{zhang2018blind} & 0.627 & 0.688 & 0.657 & 0.390 & 0.435 & 0.412 & 0.654 & 0.664 & 0.659 & 0.470 & 0.460 & 0.465 & 0.567 \\
    HyperIQA\cite{Su_2020_CVPR} & 0.657 & 0.692 & 0.674 & 0.418 & 0.465 & 0.441 & 0.669 & 0.672 & 0.670 & 0.464 & 0.431 & 0.447 & 0.582 \\
    TReS\cite{golestaneh2022no} & 0.646 & 0.702 & 0.674 & 0.445 & 0.488 & 0.466 & 0.650 & 0.654 & 0.652 & 0.505 & 0.483 & 0.494 & 0.595 \\
    Re-IQA\cite{saha2023re} & 0.473 & 0.352 & 0.412 & 0.243 & 0.154 & 0.198 & 0.654 & 0.654 & 0.654 & 0.479 & 0.484 & 0.481 & 0.486 \\
    AMFF-Net\cite{zhou2024adaptive} & 0.654 & 0.695 & 0.674 & 0.554 & 0.624 & 0.589 & 0.678 & 0.669 & 0.673 & 0.546 & \underline{0.549} & \underline{0.547} & 0.618 \\
    \midrule
    TIER\cite{yuan2024tier} & 0.653 & 0.672 & 0.663 & 0.453 & 0.470 & 0.462 & 0.667 & 0.679 & 0.673 & 0.465 & 0.466 & 0.466 & 0.566 \\
    CLIP-IQA\cite{wang2023exploring}& 0.614 & 0.665 & 0.640 & 0.487 & 0.544 & 0.516 & 0.689 & 0.718 & 0.704 & 0.526 & 0.490 & 0.508 & 0.592 \\
    CLIP-AGIQA\cite{tang2025clip} & 0.660 & 0.679 & 0.670 & - & - & - & 0.562 & 0.569 & 0.566 & - & - & - & 0.618 \\
    IPCE\cite{peng2024aigc}  & 0.752 & \underline{0.809} & \underline{0.781} & \underline{0.655} & \underline{0.730} & \underline{0.693} & \underline{0.764} & 0.755 & \underline{0.760} & \underline{0.608} & \underline{0.602} & \underline{0.605} & 0.709 \\
    \midrule
    AGQG +TIER & 0.671 & 0.716 & 0.694 & 0.476 & 0.500 & 0.488 & 0.687 & 0.698 & 0.693 & 0.500 & 0.478 & 0.489 & 0.591 \\
    \rowcolor[HTML]{FFEEED}
    \textit{Improvement}
  & $\uparrow\best{2.8\%}$  & $\uparrow\best{6.5\%}$  & $\uparrow\best{4.7\%}$  & $\uparrow\best{5.1\%}$  & $\uparrow\best{6.4\%}$  & $\uparrow\best{5.6\%}$  & $\uparrow\best{3.0\%}$  & $\uparrow\best{2.8\%}$  & $\uparrow\best{3.0\%}$  & $\uparrow\best{7.5\%}$ & $\uparrow\best{2.6\%}$  & $\uparrow\best{4.9\%}$ & $\uparrow\best{4.4\%}$ \\
    AGQG +CLIP-IQA & 0.672 & 0.727 & 0.700 & 0.580 & 0.651 & 0.616 & 0.733 & 0.748 & 0.741 & 0.576 & 0.543 & 0.560 & 0.654 \\
    \rowcolor[HTML]{FFEEED}
    \textit{Improvement}
  & $\uparrow\best{9.4\%}$ & $\uparrow\best{9.3\%}$ & $\uparrow\best{9.4\%}$ & $\uparrow\best{19.1\%}$ & $\uparrow\best{19.7\%}$ & $\uparrow\best{19.4\%}$ & $\uparrow\best{6.4\%}$ & $\uparrow\best{4.2\%}$ & $\uparrow\best{5.3\%}$ & $\uparrow\best{9.5\%}$ & $\uparrow\best{10.8\%}$ & $\uparrow\best{10.2\%}$ & $\uparrow\best{10.5\%}$ \\
    AGQG +CLIP-AGIQA & \underline{0.757} & 0.791 & 0.774 & -  & - & - & 0.753 & \underline{0.763} & 0.758 & - & - & - & \textbf{0.766} \\
    \rowcolor[HTML]{FFEEED}
    \textit{Improvement}
  & $\uparrow\best{14.7\%}$ & $\uparrow\best{16.5\%}$ & $\uparrow\best{15.5\%}$ & - & - & - & $\uparrow\best{34.0\%}$ & $\uparrow\best{34.1\%}$ & $\uparrow\best{34.0\%}$ & - & - & - & $\uparrow\best{23.9\%}$ \\
    AGQG +IPCE & \textbf{0.767} & \textbf{0.826} & \textbf{0.797} & \textbf{0.663} & \textbf{0.741} & \textbf{0.702} & \textbf{0.778} & \textbf{0.780} & \textbf{0.779} & \textbf{0.626} & \textbf{0.619} & \textbf{0.623} & \underline{0.725} \\
    \rowcolor[HTML]{FFEEED}
    \textit{Improvement}
  & $\uparrow\best{2.0\%}$  & $\uparrow\best{2.1\%}$  & $\uparrow\best{2.0\%}$  & $\uparrow\best{1.2\%}$ & $\uparrow\best{1.5\%}$ & $\uparrow\best{1.3\%}$ & $\uparrow\best{1.8\%}$  & $\uparrow\best{3.3\%}$  & $\uparrow\best{2.5\%}$  & $\uparrow\best{3.0\%}$  & $\uparrow\best{2.8\%}$ & $\uparrow\best{2.9\%}$ & $\uparrow\best{2.3\%}$ \\
\shline
\end{tabular}}
\end{table*}

\begin{table*}[!t]
\renewcommand{\arraystretch}{1.1}
\centering
\caption{Memory usage, FLOPs, and parameter size for different AGIQA models with and without our AGQG integration. The average SRCC and PLCC performance gains across intra- and cross-dataset settings are also reported for better comparison.}
\vspace{-10pt}
\label{tab:reso_analysis}
\resizebox{0.8\textwidth}{!}{
\begin{tabular}{l|cccc|cccc}
\shline
\rowcolor[HTML]{EFEFEF} 
\multicolumn{1}{c|}{\cellcolor[HTML]{EFEFEF}} & \multicolumn{2}{c}{Intra-datasets} & \multicolumn{2}{c|}{Cross-datasets}   & &  &  \\ 
\hhline{>{\arrayrulecolor[HTML]{EFEFEF}}->{\arrayrulecolor{black}}----|>{\arrayrulecolor{black}}}
\rowcolor[HTML]{EFEFEF} 
\multicolumn{1}{l|}{\multirow{-2}{*}{\cellcolor[HTML]{EFEFEF}Methods}} & SRCC & PLCC & SRCC & PLCC &\multicolumn{1}{l}{\multirow{-2}{*}{\cellcolor[HTML]{EFEFEF}Memory (MB)}} &\multicolumn{1}{l}{\multirow{-2}{*}{\cellcolor[HTML]{EFEFEF}FLOPs (G)}} &\multicolumn{1}{l}{\multirow{-2}{*}{\cellcolor[HTML]{EFEFEF}Parameters (M)}} \\ \hline \hline
TIER\cite{yuan2024tier} & 0.7601 & 0.7937 & 0.5595 & 0.5718  & 522.66 & 38.29 & 136.63 \\
CLIP-IQA\cite{wang2023exploring} & 0.7519 & 0.7703 & 0.5790 & 0.6043  & 389.33 & 24.80 & 48.72 \\
CLIP-AGIQA\cite{tang2025clip} & 0.8154 & 0.8354 & 0.6110 & 0.6240  & 584.83 & 58.30 & 170.29 \\
IPCE\cite{peng2024aigc} & 0.8298 & 0.8612 & 0.6948 & 0.7240 & 337.06 & 36.21 & 84.23 \\
\hline
\rowcolor{mygray}
AGQG +TIER & 0.8267 & 0.8207 & 0.5835 & 0.5980& $524.17$ & $38.30$ & $137.03$ \\
\rowcolor[HTML]{FFEEED}
\textit{Variation} & $\uparrow\best{8.76\%}$ & $\uparrow\best{3.40\%}$ & $\uparrow\best{4.29\%}$ & $\uparrow\best{4.58\%}$  & $\uparrow 0.29\%$ & $\uparrow 0.02\%$ & $\uparrow 0.29\%$ \\
\rowcolor{mygray}
AGQG +CLIP-IQA & 0.8130 & 0.8446 & 0.6403 & 0.6673  & $390.46$ & $24.81$ & $49.02$ \\
\rowcolor[HTML]{FFEEED}
\textit{Variation} & $\uparrow\best{8.13\%}$ & $\uparrow\best{9.65\%}$ & $\uparrow\best{10.59\%}$ & $\uparrow\best{10.43\%}$  & $\uparrow 0.29\%$ & $\uparrow 0.03\%$ & $\uparrow 0.62\%$ \\
\rowcolor{mygray}
AGQG +CLIP-AGIQA & 0.8535 & 0.8735 & 0.7550 & 0.7770  & $573.54$ & $52.83$ & $165.60$ \\
\rowcolor[HTML]{FFEEED}
\textit{Variation} & $\uparrow\best{4.67\%}$ & $\uparrow\best{4.56\%}$ & $\uparrow\best{23.57\%}$ & $\uparrow\best{24.52\%}$   &  $\downarrow 1.93\%$ & $\downarrow 9.38\%$ & $\downarrow 2.75\%$ \\
\rowcolor{mygray}
AGQG +IPCE & 0.8345 & 0.8627 & 0.7085 & 0.7415  & $337.70$ & $28.50$ & $84.39$ \\
\rowcolor[HTML]{FFEEED}
\textit{Variation} & $\uparrow\best{0.57\%}$ & $\uparrow\best{0.17\%}$ & $\uparrow\best{1.97\%}$ & $\uparrow\best{2.42\%}$  & $\uparrow 0.19\%$ & $\downarrow 27.05\%$ & $\uparrow 0.19\%$ \\
\shline
\end{tabular}
}
\end{table*}

\begin{figure*}[!t]
\centering
\includegraphics[width=\textwidth]{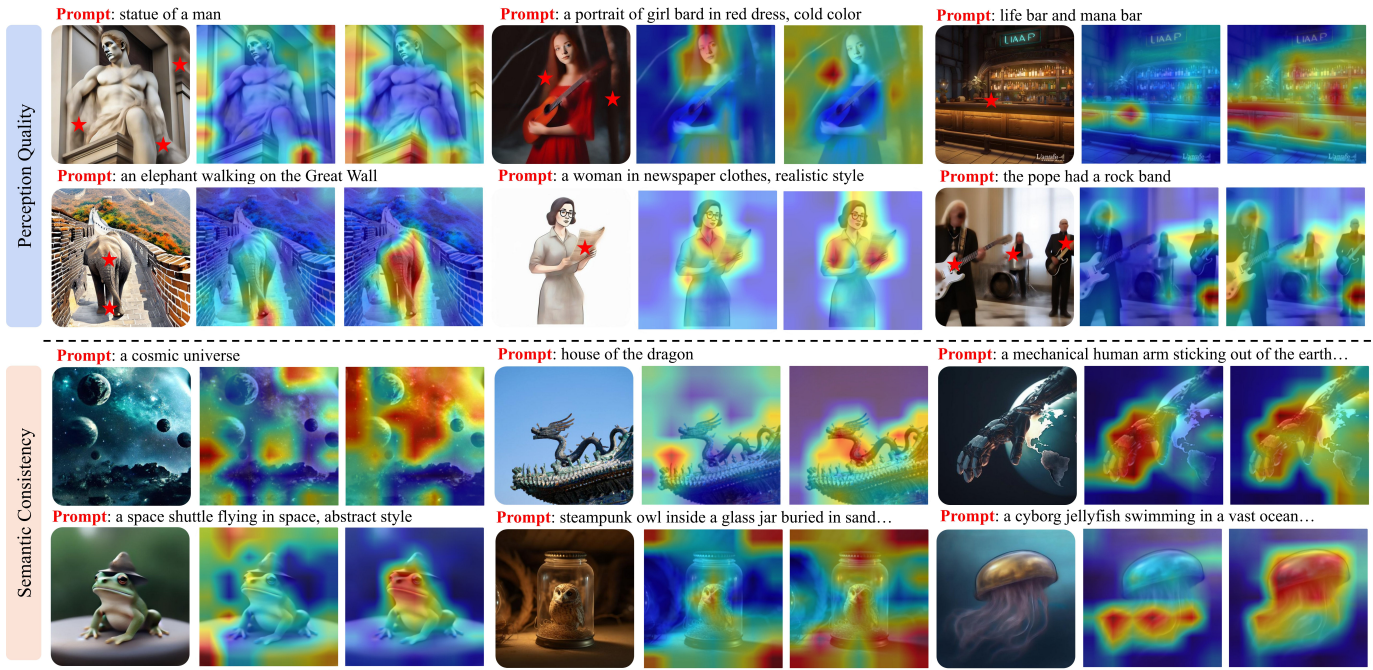}
\vspace{-20pt}
\caption{Visualization of attention maps for the perceptual quality and semantic consistency assessment tasks.  In each row, the 1st, 4th, and 7th columns display the sampled AGIs; the 2nd, 5th, and 8th columns show attention maps produced by the original IPCE model; and the 3rd, 6th, and 9th columns present attention maps from the IPCE model enhanced with our AGQG  module. Red stars (\textcolor{red}{$\bigstar$}) mark regions of significant distortion under the perceptual quality assessment task. Zoomed-in for better detail.}
\label{fig:atten}
\end{figure*}

\subsection{Generalization on Classical IQA Frameworks}
\label{subsec:generalization}

To evaluate the generalization capability of AGQG, we further integrate it into two classical IQA frameworks, DBCNN~\cite{zhang2018blind} and HyperIQA~\cite{Su_2020_CVPR}. We examine the performance improvements across both natural images and screen content images on the TID2013~\cite{ponomarenko2015image} dataset and the SIQAD~\cite{yang2015perceptual} dataset, respectively. The results are presented in Fig.~\ref{fig:trad_iqa}. It can be observed that the incorporation of AGQG leads to consistent performance gains across both architectures. 
On TID2013, ``AGQG+DBCNN" and ``AGQG+HyperIQA" achieve the average SRCC improvements of 2.49\%, with PLCC increases of 2.78\%.  On SIQAD, the average SRCC improves by 4.70\%, with a corresponding average PLCC gain of 4.13\%.
These results demonstrate that AGQG serves as a versatile plug-and-play module capable of enhancing various IQA architectures without substantial structural modification.

\subsection{Resource Efficiency Analysis}
We evaluate the computational efficiency of AGIQA models with and without AGQG integration in terms of memory usage, FLOPs, and parameter size. Table~\ref{tab:reso_analysis} summarizes the results across intra- and cross-dataset evaluations.

The integration of AGQG achieves the best of both worlds: it simultaneously improves prediction performance and reduces computational complexity. For example, for IPCE, FLOPs decrease by 27.05\% while cross-dataset SRCC and PLCC increase by 1.97\% and 2.42\%, respectively. Similarly, CLIP-AGIQA exhibits a 9.38\% reduction in FLOPs and a 2.75\% reduction in parameters, accompanied by the largest cross-dataset performance gains among all methods. These improvements arise from the use of fixed, task-specific textual templates in AGQG, which streamline feature extraction and eliminate the reliance on learned or hand-crafted prompts, while maintaining high prediction accuracy. Across all evaluated frameworks, the parameter overhead introduced by AGQG is negligible, ranging from 0.29\% for TIER to 0.62\% for CLIP-AGIQA, while SRCC and PLCC improvements remain substantial. For instance, CLIP-AGIQA with AGQG achieves over 23\% improvement in cross-dataset SRCC, indicating that performance gains far exceed the minor increments in parameters and inference time. These results demonstrate that AGQG provides a dual advantage of enhanced quality assessment performance and greater computational efficiency, highlighting its effectiveness in both intra- and cross-dataset scenarios.

\subsection{Visualization of Attention Distribution}
To further validate the effectiveness of our AGQG  module, we compare the attention distributions of the advanced IPCE model with and without AGQG  integration, as illustrated in Fig.~\ref{fig:atten}. Specifically, we respectively visualize the spatial attention maps generated during the perceptual quality assessment and prompt alignment tasks. From the figure, we could observe that the IPCE model enhanced with our AGQG  module allocates significantly more attention to distorted regions when evaluating the perceptual quality. For example, the blurred arms of the statue in the image at row 1, column 1, and the unnatural background textures in the image at row 1, column 4. In comparison, these regions are largely overlooked by the original IPCE model. In the prompt alignment task, the introduction of AGQG  also demonstrates a sharper focus on key semantic areas, such as the dragon’s body (row 3, column 6) and the jellyfish with the surrounding seafloor background (row 4, column 9), resulting in more accurate alignment predictions.

\begin{table*}[htbp]
\centering
\renewcommand{\arraystretch}{1.2}
\caption{Performance comparison with different numbers of quality grades.}
\vspace{-10pt}
\label{tab:ablation_rating_levels_and_loss}
\resizebox{1.0\textwidth}{!}{
\begin{tabular}{l|c|cc|cc|cc|cc|cc|cc|cc}
\shline
\rowcolor[HTML]{EFEFEF}
\multicolumn{2}{c|}{\cellcolor[HTML]{EFEFEF}}  
& \multicolumn{6}{c|}{\cellcolor[HTML]{EFEFEF}AIGCIQA2023} 
& \multicolumn{4}{c|}{\cellcolor[HTML]{EFEFEF}AIGCIQA2023 $\rightarrow$ AGIQA-3K} 
& \multicolumn{4}{c}{\cellcolor[HTML]{EFEFEF}AGIQA-3K $\rightarrow$ AIGCIQA2023} \\

\hhline{>{\arrayrulecolor[HTML]{EFEFEF}}-->{\arrayrulecolor{black}}--------------}
\rowcolor[HTML]{EFEFEF}
\multicolumn{2}{c|}{\cellcolor[HTML]{EFEFEF}}  
& \multicolumn{2}{c|}{Quality} 
& \multicolumn{2}{c|}{Consistency} 
& \multicolumn{2}{c|}{Authenticity}
& \multicolumn{2}{c|}{Quality} 
& \multicolumn{2}{c|}{Consistency}
& \multicolumn{2}{c|}{Quality} 
& \multicolumn{2}{c}{Consistency} \\

\hhline{>{\arrayrulecolor[HTML]{EFEFEF}}-->{\arrayrulecolor{black}}--------------}
\rowcolor[HTML]{EFEFEF}
\multicolumn{2}{c|}{\multirow{-3}{*}{\cellcolor[HTML]{EFEFEF}Level}} 
& SRCC & PLCC 
& SRCC & PLCC 
& SRCC & PLCC 
& SRCC & PLCC 
& SRCC & PLCC 
& SRCC & PLCC 
& SRCC & PLCC \\

\hline \hline
(1) & L3  & 0.8665 & 0.8840 & 0.8048 & 0.7960 & 0.8096 & 0.8021 
      & 0.7721 & 0.7678 & 0.5785 & 0.5746  
      & 0.7579 & 0.8167 & 0.6675 & 0.7395 \\

(2) & L7  & 0.8651 & 0.8835 & \textbf{0.8059} & 0.7974 & \textbf{0.8119} & \textbf{0.8035} 
      & 0.7611 & \textbf{0.7852} & \underline{0.5874} & \underline{0.5820}  
      & \textbf{0.7756} & \underline{0.8027} & \underline{0.6725} & \underline{0.7507} \\

(3) & L9  & \underline{0.8670} & \textbf{0.8851} & \underline{0.8056} & \underline{0.7976} & 0.8079 & 0.8008 
      & 0.7559 & 0.7748 & 0.5825 & 0.5724  
      & 0.7522 & 0.7871 & \textbf{0.6755} & \textbf{0.7578} \\

\midrule
\rowcolor[HTML]{FFEEED}
(4) & \textbf{L5 (Ours)} 
      & \textbf{0.8671} & \underline{0.8846} & 0.8052 & \textbf{0.7978} & \underline{0.8115} & \underline{0.8031} 
      & \textbf{0.7784} & \underline{0.7795} & \textbf{0.6261} & \textbf{0.6192}  
      & \underline{0.7666} & \textbf{0.8260} & 0.6633 & 0.7410 \\

\shline
\end{tabular}
}
\end{table*}

\subsection{Ablation Studies}  
The core of the proposed AGQG module lies in the \textit{difficulty estimation branch}. 
Specifically, both image and text features are jointly exploited to estimate the priors $\beta_1$ and $\gamma$, 
which are subsequently refined by a Temperatured TeLU layer to categorize samples into five discrete quality grades. 
To identify the optimal configuration, we perform extensive ablation studies examining three key factors: 
(1) the modality fusion strategy with our proposed image-dependent temperature mechanism, (2) the activation function used in the difficulty branch, and 
(3) the number of quality grades. 
All experiments are conducted based on the strongest baseline, IPCE.\\

\noindent\textbf{Study of Modality Combination for Difficulty Estimation.}  
As illustrated in Fig.~\ref{fig:frame}, both the image and text prompt are utilized for quality difficulty estimation. To evaluate the importance of this modality combination, we investigate three simplified variants:  
(1) \textit{Image-only}: The text feature is ablated, and $F_T$ in Eqn.~(\ref{enphi}) is replaced with $F_I$ for regressing the difficulty parameters $\beta_1$ and $\gamma$;  {(2) \textit{Text-only}: We replace $F_I$ in Eqn.~(\ref{enphi}) with $F_T$ to ablate image feature.} (3) \textit{Temp-abla}: The text-derived difficulty priors $\beta_1^T$ and $\gamma^T$ are directly used as $\beta_1$ and $\gamma$, \aka, the temperature dependence is ablated. The comparison results, presented in Table~\ref{tab:mod}, show that either modality alone provides a performance improvement over the original IPCE model. Nevertheless, the combination of both modalities yields the highest average performance, demonstrating that the integration of visual and textual information leads to more accurate, image-specific difficulty estimation. \\

\begin{figure}[!t]
\centering
\includegraphics[width=0.5 \textwidth]{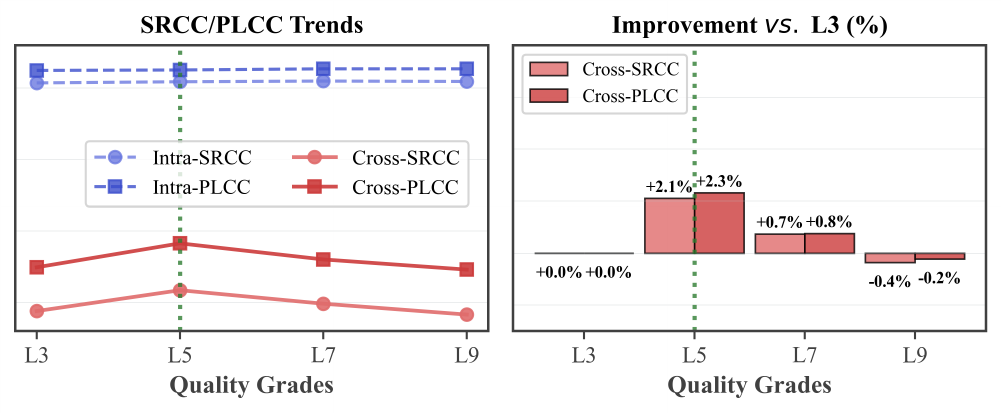}
\vspace{-20pt}
\caption{Left: Performance comparison across different quality grades (L3, L5, L7, L9) for both intra-dataset and cross-dataset settings. Right: Percentage improvement in cross-dataset performance relative to the L3 setting.}
\label{fig:level}
\end{figure}

\begin{table}[htp]
\centering
\setlength{\tabcolsep}{5pt}
\renewcommand{\arraystretch}{1.23}
\caption{Ablation study on modality combinations for difficulty estimation. Perceptual quality and prompt consistency are denoted as Qual. and Cons., respectively.} 
\vspace{-10pt}
\resizebox{0.5\textwidth}{!}{
\begin{tabular}{l|cc|cc|cc|c}
\shline
\rowcolor[HTML]{EFEFEF}
\multicolumn{1}{c|}{\cellcolor[HTML]{EFEFEF}} & \multicolumn{2}{c|}{\cellcolor[HTML]{EFEFEF}AGIQA-1K} & \multicolumn{4}{c|}{\cellcolor[HTML]{EFEFEF}AGIQA-3K$\rightarrow$AIGCIQA2023} & \multicolumn{1}{c}{\cellcolor[HTML]{EFEFEF}} \\
\hhline{>{\arrayrulecolor[HTML]{EFEFEF}}->{\arrayrulecolor{black}}------}
\rowcolor[HTML]{EFEFEF}
\multicolumn{1}{c|}{\cellcolor[HTML]{EFEFEF}} & \multicolumn{2}{c|}{Qual.} & \multicolumn{2}{c|}{Qual.} & \multicolumn{2}{c|}{Cons.} & \multicolumn{1}{c}{\cellcolor[HTML]{EFEFEF}} \\
\hhline{>{\arrayrulecolor[HTML]{EFEFEF}}->{\arrayrulecolor{black}}------}
\rowcolor[HTML]{EFEFEF}
\multicolumn{1}{c|}{\multirow{-3}{*}{\cellcolor[HTML]{EFEFEF}Methods}} & SRCC & PLCC & SRCC & PLCC & SRCC & PLCC & \multicolumn{1}{c}{\multirow{-3}{*}{\cellcolor[HTML]{EFEFEF}Avg.}} \\
\hline \hline
Image-only & 0.8578 & 0.8882 & 0.7582 & 0.7781 & 0.5976 & 0.5949 & \underline{0.7458} \\
Text-only & 0.8595 & \textbf{0.8897} & \underline{0.7651} & \underline{0.7693} & 0.5934 & 0.5927 & 0.7450 \\
Temp-abla & \textbf{0.8612} & \underline{0.8895} & 0.7611 & 0.7637 & 0.5925 & 0.5846 & 0.7421 \\
IPCE~\cite{peng2024aigc} & 0.8535 & 0.8792 & 0.7645 & 0.7553 & \underline{0.6077} & \underline{0.6018} & 0.7437 \\
\midrule
\rowcolor[HTML]{FFEEED}
\textbf{Ours} & \underline{0.8606} & 0.8881 & \textbf{0.7784} & \textbf{0.7795} & \textbf{0.6261} & \textbf{0.6192} & \textbf{0.7587} \\
\shline
\end{tabular}
}
\label{tab:mod}
\end{table}

\noindent\textbf{Study of  Activation Function.}  
\label{sec:activation_ablation}
To validate the design choice of the Temperatured activation function, we compare TeLU against three widely used activations: Sigmoid, ReLU, and Softplus. As shown in Table~\ref{tab:act}, all activations yield competitive results on the in-domain AIGCIQA2023 dataset, with SRCC and PLCC differences within 0.5\%. However, TeLU exhibits clear superiority in the more challenging cross-domain setting (AGIQA-3K~$\rightarrow$~AIGCIQA2023), achieving approximately 6.1\% improvements in both SRCC and PLCC over the second-best baseline. This can be attributed to its unique design: the $\tanh(\exp(\cdot))$ composition provides a smooth, bounded nonlinear transformation that effectively modulates text-derived difficulty priors, enhancing the integration of textual and image-specific cues.\\
 
\begin{table}[htbp]
\centering
\setlength{\tabcolsep}{4pt}
\renewcommand{\arraystretch}{1.25}
\caption{Performance comparison under different activation function.}
\vspace{-10pt}
\label{tab:act}
\resizebox{0.5\textwidth}{!}{
\begin{tabular}{l|cc|cc|cc|cc}
\shline
\rowcolor[HTML]{EFEFEF}
\multicolumn{1}{c|}{\cellcolor[HTML]{EFEFEF}}  & \multicolumn{4}{c|}{\cellcolor[HTML]{EFEFEF}AIGCIQA2023} &\multicolumn{4}{c}{\cellcolor[HTML]{EFEFEF}AGIQA-3K$\rightarrow$ AIGCIQA2023}\\

\hhline{>{\arrayrulecolor[HTML]{EFEFEF}}->{\arrayrulecolor{black}}--------}
\rowcolor[HTML]{EFEFEF}
\multicolumn{1}{c|}{\cellcolor[HTML]{EFEFEF}} & \multicolumn{2}{c|}{Qual.} & \multicolumn{2}{c|}{Cons.} & \multicolumn{2}{c|}{Qual.} & \multicolumn{2}{c}{Cons.}  \\

\hhline{>{\arrayrulecolor[HTML]{EFEFEF}}->{\arrayrulecolor{black}}--------}
\rowcolor[HTML]{EFEFEF}
\multicolumn{1}{c|}{\multirow{-3}{*}{\cellcolor[HTML]{EFEFEF}Act.}} & SRCC & PLCC & SRCC & PLCC & SRCC & PLCC & SRCC & PLCC  \\
\hline \hline
Sigmoid  & \underline{0.8678} & \textbf{0.8863} & \underline{0.8065} & \underline{0.7987} & 0.7583 & 0.7657 & \underline{0.5900} & \underline{0.5834} \\
ReLU  & 0.8658 & 0.8832 & 0.8042 & 0.7967 & 0.7572 & \underline{0.7788} & 0.5865 & 0.5791 \\
Softplus  & \textbf{0.8679} & \underline{0.8861} & \textbf{0.8082} & \textbf{0.7992} & \underline{0.7780} & 0.7787 & 0.5762 & 0.5778 \\
\midrule
\rowcolor[HTML]{FFEEED}
\textbf{Ours}  & 0.8671 & 0.8846 & 0.8052 & 0.7978 & \textbf{0.7784} & \textbf{0.7795} & \textbf{0.6261} & \textbf{0.6192} \\
\shline
\end{tabular}
}
\end{table}

\noindent\textbf{Study of the Number of Quality Grades.} In our AGQG  module, five quality grades ($\beta_1$, $\beta_2$, ..., $\beta_4$) are used for quality categorization. To investigate the optimal number of quality levels, we experiment with three alternative configurations: three, seven, and nine grades, denoted as L3, L7, and L9, respectively, alongside the original five-grade setting (L5). These configurations are evaluated in terms of their impact on model performance across three key criteria: perceptual quality, prompt correspondence and image authenticity on the AGIQA2023. Besides we evaluate model generalization explored by training on one dataset and testing on another between AGIQA-3k and AGIQA2023. The results are summarized in Table~\ref{tab:ablation_rating_levels_and_loss} and Fig.~\ref{fig:level}. The comparison results reveal that the L5 setting achieves the highest PLCC for the perceptual quality on AIGCIQA2023 and presents the best generalization ability when training on AGIQA-3k. Although L3 and L7 slightly outperform L5 in certain quality aspects, L5 yields the highest average performance across all evaluation criteria. Increasing the number of quality levels from L3 to L5 leads to notable cross-dataset performance gains. In contrast, further increasing granularity (e.g., L7 or L9) provides only marginal improvements in intra-dataset performance while simultaneously impairing cross-domain generalization. These observations suggest that L5 represents a balanced trade-off between assessment granularity and model complexity. Moreover, this setting aligns well with real-world scoring practices\cite{series2012methodology}, where human evaluators typically rely on a limited set of discrete quality levels to ensure both interpretability and cognitive efficiency.

\section{Conclusion}
In this paper, we introduce AGQG, a novel quality-level categorization module for AI-generated image quality assessment, drawing inspiration from the Graded Response Model in psychometrics. Our approach tackles the persistent challenge of semantic drift in existing AGIQA frameworks, where unreliable text–image embedding alignment often leads to unstable, multi-modal prediction behaviors. By decoupling image quality assessment into image ability and difficulty modeling and incorporating a mathematically principled monotonic design, AGQG achieves stable and interpretable predictions that are better aligned with human perceptual judgment. Comprehensive experiments across multiple state-of-the-art AGIQA models and diverse datasets demonstrate that AGQG functions as a plug-and-play enhancement, delivering consistent performance improvements without imposing substantial computational overhead. Notably, its strong cross-dataset generalization highlights its capability to transcend dataset-specific biases, positioning AGQG as a foundational building block for future AGIQA system design. This work thus bridges psychometric theory and vision-language modeling, paving the way for more reliable, interpretable, and human-centric evaluation of AI-generated visual content.

\proofsection{Theoretical Proofs}\label{sec:proofs}

\section*{Unimodal Distribution Analysis}
\label{sec:unimodal_theorem}

We formally establish the unimodality of the AGRM category probabilities through a sequence of carefully constructed lemmas culminating in our main theorem. Our theoretical investigation begins by proving the fundamental shift property of intermediate category probabilities, which emerges directly from the arithmetic progression of difficulty parameters. We further introduce the critical threshold condition \(\gamma > \frac{2 \ln 2}{D \alpha}\) to properly constrain the intersection positions of boundary category probabilities, ensuring they do not disrupt the global unimodal structure. Finally, we systematically integrate these results to demonstrate that the complete probability distribution maintains strict unimodality across all ability levels, with a unique modal category and monotonic decay on both sides

\subsection*{Definitions}
Our AGRM formally defines the category probability for \(k \in \{1,\dots,n\}\) as
\begin{equation}
\label{eq:pk_def}
P_k(\theta) = 
\begin{cases}
1 - \sigma(D\alpha(\theta-\beta_1)), & k=1,\\
\sigma(D\alpha(\theta-\beta_{k-1})) - \sigma(D\alpha(\theta-\beta_k)), & 2 \le k \le n-1,\\
\sigma(D\alpha(\theta-\beta_{n-1})), & k=n,
\end{cases}
\end{equation}
with logistic function $\sigma(x) = 1/(1+e^{-x})$, $D,\alpha > 0$ are scaling parameters, and $\beta_k = \beta_1 + (k-1)\gamma$ with \(\gamma > \frac{2 \ln 2}{D \alpha}\).

\subsection*{Lemmas}

\begin{lemma}[Peak of Intermediate Categories]
\label{lem:peak}
For \(2 \le k \le n-1\), \(P_k(\theta)\) attains a unique maximum at
\[
\theta^*_k = \frac{\beta_{k-1} + \beta_k}{2}.
\]
Moreover, \(P_k(\theta)\) is strictly increasing for \(\theta < \theta^*_k\) and strictly decreasing for \(\theta > \theta^*_k\).
\end{lemma}

\begin{proof}
Taking derivative:
\[
P_k'(\theta) = \sigma'(D\alpha(\theta-\beta_{k-1})) - \sigma'(D\alpha(\theta-\beta_k)).
\]
Setting \( P_k'(\theta) = 0 \) yields the condition \( \sigma'(D\alpha(\theta^* - \beta_{k-1})) = \sigma'(D\alpha(\theta^* - \beta_k)) \), which implies that the critical point satisfies \( \theta_k^* = (\beta_{k-1} + \beta_k)/2 \). When \( \theta' < \theta^* \), we have \( P_k' > 0 \), indicating that \( P_k(\theta) \) is strictly increasing on this interval; conversely, when \( \theta' > \theta^* \), \( P_k' < 0 \), implying that \( P_k(\theta) \) is monotonically decreasing. This opposite monotonic behavior of \( P_k'(\theta) \) on either side of \( \theta^* \) ensures the existence of a unique peak for \( P_k(\theta) \) at this critical point.

\end{proof}

\begin{figure}[!t]
\centering
\includegraphics[width=0.5\textwidth]{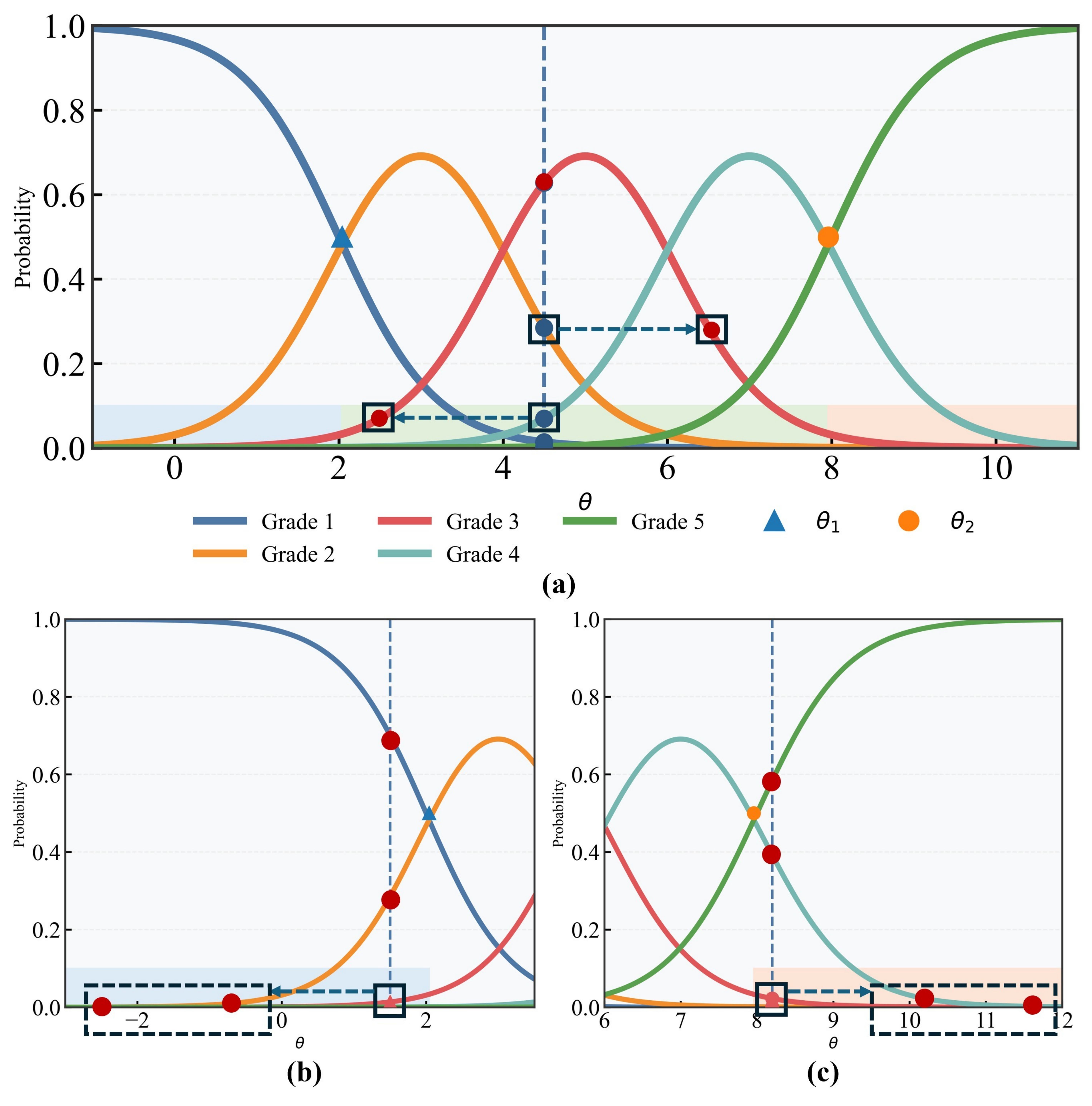}
\vspace{-20pt}
\caption{The domain of \( \theta \) can be divided into three parts \( (-\infty, \theta_1) \), \( (\theta_1, \theta_2) \), and \( (\theta_2, +\infty) \). (a) According to Lemma~\ref{lem:shift}, for \(2 \leq k \leq n-1\), \(P_{k+1}\) corresponds to a rightward translation of \(P_k\) by \(\gamma\); as illustrated by the curves, these two sets of points (\(P_k\) and \(P_{k+1}\)) are mutually convertible. (b) For $\theta < \theta_{1}$, $P_{1}$ attains the maximum over the family $\{P_{k}\}_{k=1}^{n}$, while the category probabilitie in $\{P_{k}\}_{k=3}^{n-1}$ can be transformed into $P_{n-1}$ through a rigid leftward translation.  (c) For $\theta > \theta_{2}$, $P_{n}$ attains the maximum, and each probability $P_{k}$ with $k = 2, \dots, n-2$  can be converted into $P_{2}$ by the aforementioned translation.}
\label{fig:append}
\end{figure}

\begin{lemma}[Shift Property]
\label{lem:shift}
For $2 \le k \le n-2$, the category probabilities(as illustrated in Fig.~\ref{fig:append}(a)) satisfy:
\[
P_{k+1}(\theta) = P_k(\theta - \gamma),
\]
indicating that $P_{k+1}$ is a rightward translation of $P_k$ by $\gamma$. The intersection point of \(P_k\) and \(P_{k+1}\) occurs exactly \(\gamma\) apart from the previous intersection. {Formally, the intersection point of $P_k$ and $P_{k+1}$ is denoted as $x_k$ and $x_{k+1}-x_k=\gamma$.}
\end{lemma}

\begin{proof}
By definition,
\[
P_{k+1}(\theta) = \sigma(D\alpha(\theta-\beta_k)) - \sigma(D\alpha(\theta-\beta_{k+1})).
\]
Since \(\beta_{k+1} = \beta_k + \gamma\), we can rewrite:
\begin{equation*}
\begin{aligned}
P_{k+1}(\theta) &= \sigma(D\alpha((\theta-\gamma)-\beta_{k-1})) - \sigma(D\alpha((\theta-\gamma)-\beta_k)) \\
                &= P_k(\theta-\gamma).
\end{aligned}
\end{equation*}
\end{proof}

\begin{corollary}[Unimodal Structure Induced by the Shift Property]
\label{cor:unimodal}
If the shift property \( P_{k+1}(\theta) = P_k(\theta - \gamma) \) holds for \( 2 \le k \le n-1 \), then the sequence \(\{P_k(\theta)\}\) exhibits a strictly unimodal structure:
\begin{equation*}
P_2 < P_i < P_j < P_k > P_l > P_m, \quad 2 \le i < j < k < l < m \le n-1.
\end{equation*}
\end{corollary}

\begin{proof}
Let \( k \) denote the category where \( P_k(\theta) \) attains its maximum value, and consider indices satisfying \( 2 \le i < j < k < l < m \le n-1 \).

\begin{itemize}
    \item For \( i \) and \( j \), since they lie to the left of the peak (in the increasing interval), and given that \( P_j(\theta) = P_i(\theta - (j - i)\Delta) \), it follows that \( P_i < P_j < P_k \).
    
    \item For \( l \) and \( m \), since they lie to the right of the peak (in the decreasing interval), and \( P_m(\theta) = P_l(\theta - (m - l)\Delta) \), we obtain \( P_m < P_l < P_k \).
\end{itemize}
Therefore, the shift property induces a strictly unimodal configuration. Geometrically, it implies that the entire family of interior category probability curves \(\{P_k\}_{k=2}^{n-1}\) can be generated by horizontally translating a single prototype curve by integer multiples of \(\gamma\).
\end{proof}

\begin{lemma}[Boundary Intersections]
\label{lem:boundary}
Define
\[
\theta_1 = \beta_1 - \frac{1}{D\alpha} \ln(1 - 2 e^{-D\alpha\gamma}), \quad
\theta_2 = \beta_{n-2} + \frac{1}{D\alpha} \ln(e^{D\alpha\gamma} - 2).
\]
If \(\gamma > \frac{2 \ln 2}{D \alpha}\), then
\[
\theta_1 < \theta^*_2, \quad \theta_2 > \theta^*_{n-1},
\]
{ensuring that the intersection point \( \theta_1 \) of $P_1 = P_2$ occurs before the maximum point of $P_2$, and the intersection point \( \theta_2 \) of  $P_n$ and $P_{n-1}$ occurs after the maximum point of $P_n$.}
\end{lemma}

\begin{proof}
Solving \(P_1(\theta) = P_2(\theta)\) gives \(\theta_1\). By the inequality \(\gamma > 2 \ln 2 / (D\alpha)\), a simple calculation shows \(\theta_1 < \theta^*_2 = (\beta_1 + \beta_2)/2\). Similarly, \(P_n(\theta) = P_{n-1}(\theta)\) gives \(\theta_2 > \theta^*_{n-1}\). {Hence, these inequalities ensure that the global maximum occurs at a unique interior category, guaranteeing the uniqueness of the peak.}
\end{proof}

\begin{corollary}[Boundary Probability Ordering]
\label{cor:boundary_ordering}
{Under the condition $\gamma > \frac{2 \ln 2}{D\alpha}$, the boundary probabilities exhibit the following ordering relationships:}
\begin{enumerate}
    \item \textbf{First category relationship:}
    \begin{align*}
    P_1(\theta) &> P_2(\theta), \quad \text{for } \theta < \theta_1 \\
    P_1(\theta) &< P_2(\theta), \quad \text{for } \theta > \theta_1
    \end{align*}
    
    \item \textbf{Last category relationship:}
    \begin{align*}
    P_n(\theta) &< P_{n-1}(\theta), \quad \text{for } \theta < \theta_2 \\
    P_n(\theta) &> P_{n-1}(\theta), \quad \text{for } \theta > \theta_2
    \end{align*}
\end{enumerate}
\end{corollary}

\subsection*{Main Theorem}

\begin{theorem}[Unimodality of GRM Probabilities]
\label{thm:unimodal}
Under the arithmetic difficulty sequence \(\beta = \{\beta_1 + (k-1)\gamma\}_{k=1}^{n-1}\) with \(\gamma > \frac{2 \ln 2}{D\alpha}\), there exists a unique index \(k^*\) such that
\[
P_{k^*}(\theta) \ge P_k(\theta), \quad \forall k \neq k^*,
\]
and the distribution is strictly monotone on both sides of the peak:
\[
P_1 < \cdots < P_{k^*-1} < P_{k^*} > P_{k^*+1} > \cdots > P_n.
\]

\begin{itemize}
    \item {When the mode $k^* = 1$: $P_1(\theta) > P_2(\theta) > \cdots > P_n(\theta)$}
    \item When $2 \le k^* \le n-1$: $P_1(\theta) < P_2(\theta) < \cdots < P_{k^*}(\theta) > \cdots > P_{n-1}(\theta) > P_n(\theta)$
    \item When the mode $k^* = n$: $P_1(\theta) < P_2(\theta) < \cdots < P_n(\theta)$
\end{itemize}
\end{theorem}

\begin{proof}
Divide the domain of \( \theta \) into three parts \( (-\infty, \theta_1) \), \( (\theta_1, \theta_2) \), and \( (\theta_2, +\infty) \).

\textbf{(i) For \(\theta < \theta_1\):} By Lemma~\ref{lem:boundary}, \(P_1 > P_2\). Lemma~\ref{lem:shift} ensures higher \(k\) satisfy \(P_k < P_2\) (see Fig.~\ref{fig:append}(b)). Thus,
\[
P_1 > P_2 > \dots > P_n.
\]

\textbf{(ii) For \(\theta_1 < \theta < \theta_2\):} Lemma~\ref{lem:peak} and Lemma~\ref{lem:shift} guarantee a unique \(k^*\) with \(P_{k^*-1} < P_{k^*} > P_{k^*+1}\). Combining with boundary comparisons yields strict unimodality.

\textbf{(iii) For \(\theta > \theta_2\):} Lemma~\ref{lem:boundary} gives \(P_n > P_{n-1}\), and Lemma~\ref{lem:shift} ensures lower \(k\) are smaller (see Fig.~\ref{fig:append}(c)). Therefore,
\[
P_1 < P_2 < \dots < P_n.
\]
In all intervals, the distribution has a unique peak with strictly decreasing probabilities on both sides.
\end{proof}

\begin{corollary}
The condition \(\gamma > \frac{2 \ln 2}{D \alpha}\) is sufficient to guarantee the uniqueness of the global peak and the unimodal structure.
\end{corollary}

\bibliographystyle{IEEEtran}
\small{
\bibliography{ref}}

\begin{IEEEbiography}[{\includegraphics[width=1in,height=1.25in, clip,keepaspectratio]{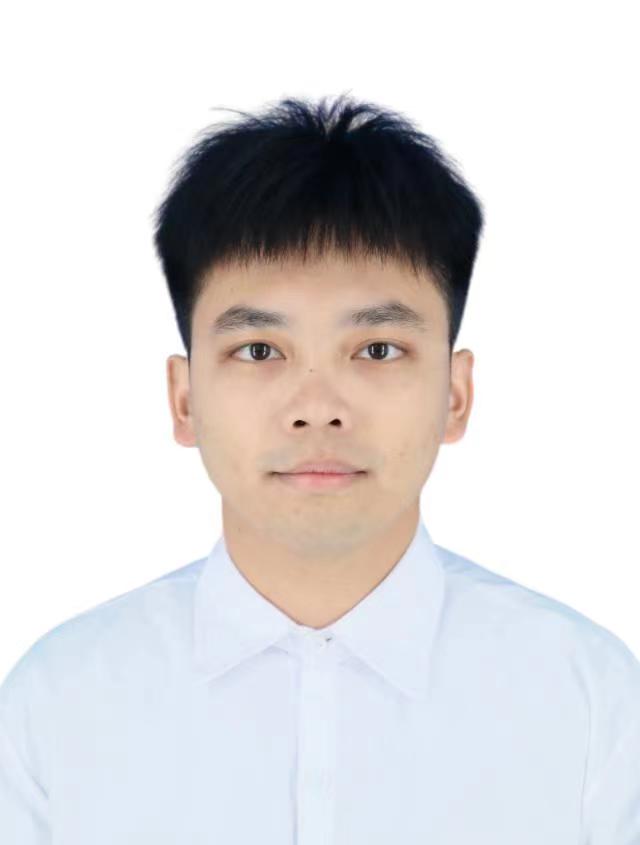}}]{Zhicheng Liao} received the B.S. degree from South China Normal University, China, in 2025. He is currently pursuing the M.S. degree with the Department of Computer Science, South China Normal University, China. His research interests include image quality assessment, image quality enhancement and diffusion model.
\end{IEEEbiography}

\begin{IEEEbiography}[{\includegraphics[width=1in,height=1.25in, clip,keepaspectratio]{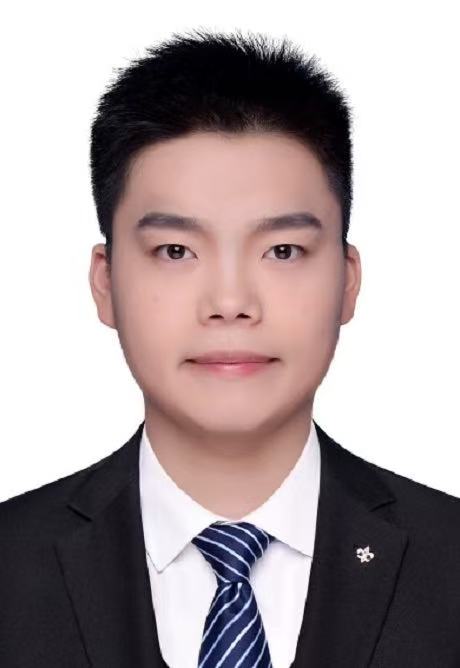}}]{Baoliang Chen} (Member, IEEE) received his B.E. degree in Electronic Information Science and Technology from Hefei University of Technology, Hefei, China, in 2015, his M.S. degree in Intelligent Information Processing from Xidian University, Xian, China, in 2018, and his Ph.D. degree in computer science from the City University of Hong Kong, Hong Kong, in 2022.  From 2022 to 2024, he was a postdoctoral researcher with the Department of Computer Science, City University of Hong Kong. He is currently an Associate Professor with the Department of Computer Science, South China Normal University.  His research interests include image/video quality assessment and transfer learning.
\end{IEEEbiography}

\begin{IEEEbiography}[{\includegraphics[width=1in,height=1.25in, clip,keepaspectratio]{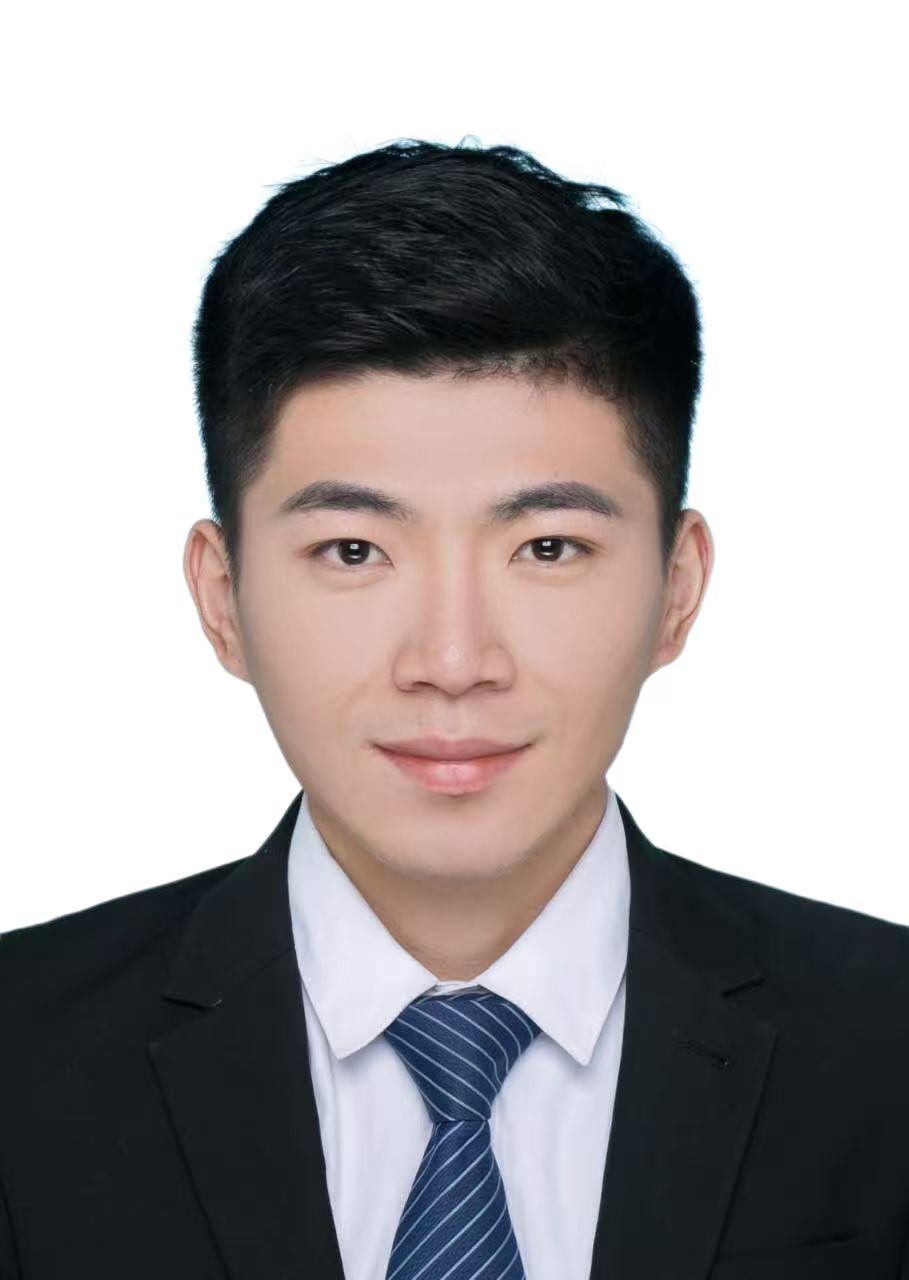}}]{Hanwei Zhu} (Member, IEEE) received the B.E and M.S. degrees from the Jiangxi University of Finance and Economics, Nanchang, China, in 2017 and 2020, respectively, and Ph.D. degree in computer science from the City University of Hong Kong, Hong Kong, in 2025.  He is currently a research scientist with the Alibaba-NTU Global e-Sustainability CorpLab (ANGEL) at Nanyang Technological University. His research interests include perceptual image processing, computational vision, and computational photography.
\end{IEEEbiography}

\begin{IEEEbiography}[{\includegraphics[width=1in,height=1.25in,clip,keepaspectratio]{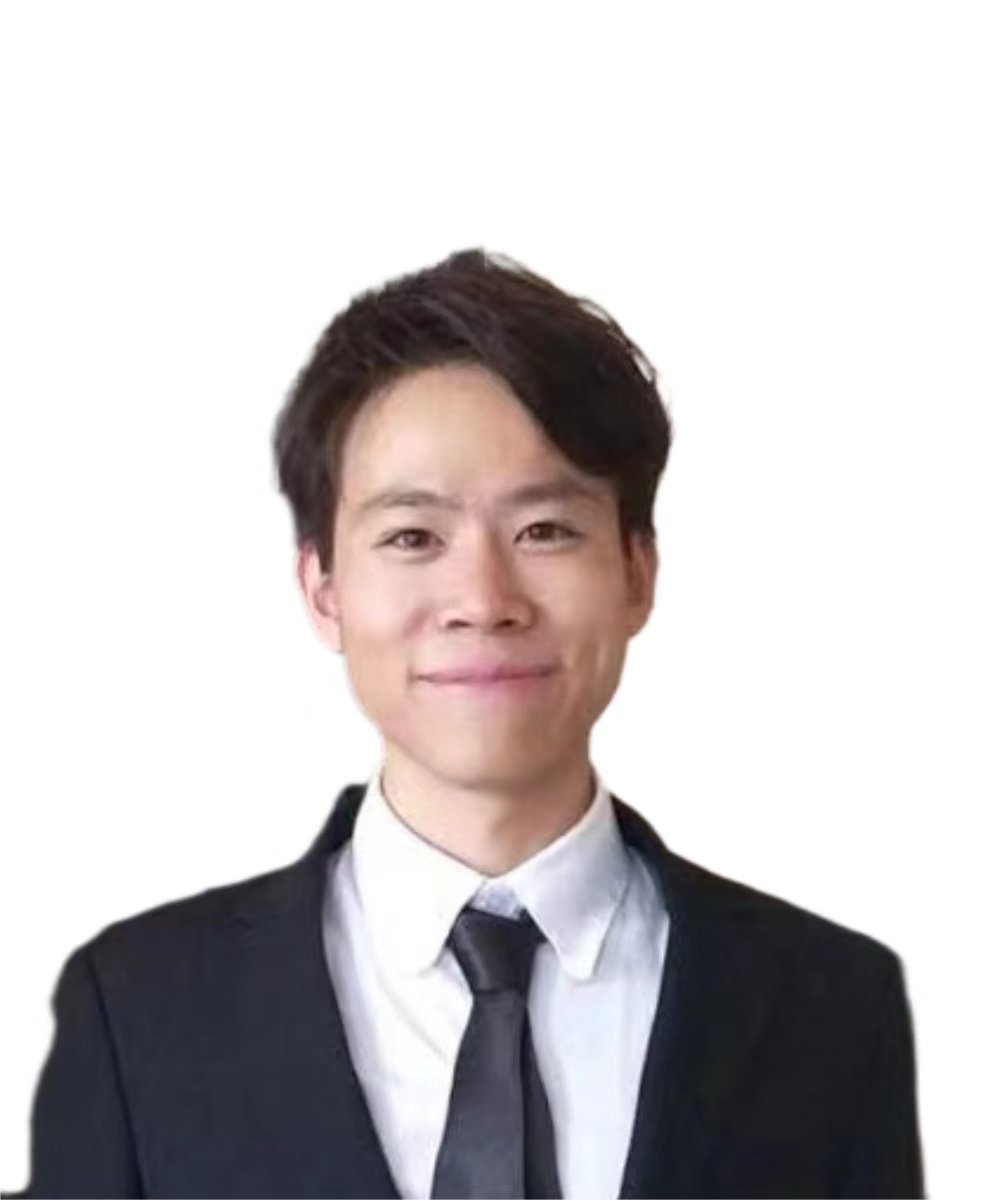}}]{Lingyu Zhu} (Member, IEEE) received the B.E. degree from Wuhan University of Technology in 2018, the M.S. degree from The Hong Kong University of Science and Technology in 2019, and the Ph.D. degree in computer science from the City University of Hong Kong, Hong Kong SAR, China, in 2024. He is currently a postdoctoral researcher with the Department of Computer Science at City University of Hong Kong. His research interests include image/video compression, image/video enhancement, image/video quality assessment, and deep learning.
\end{IEEEbiography}

\begin{IEEEbiography}[{\includegraphics[width=1in,height=1.25in,clip,keepaspectratio]{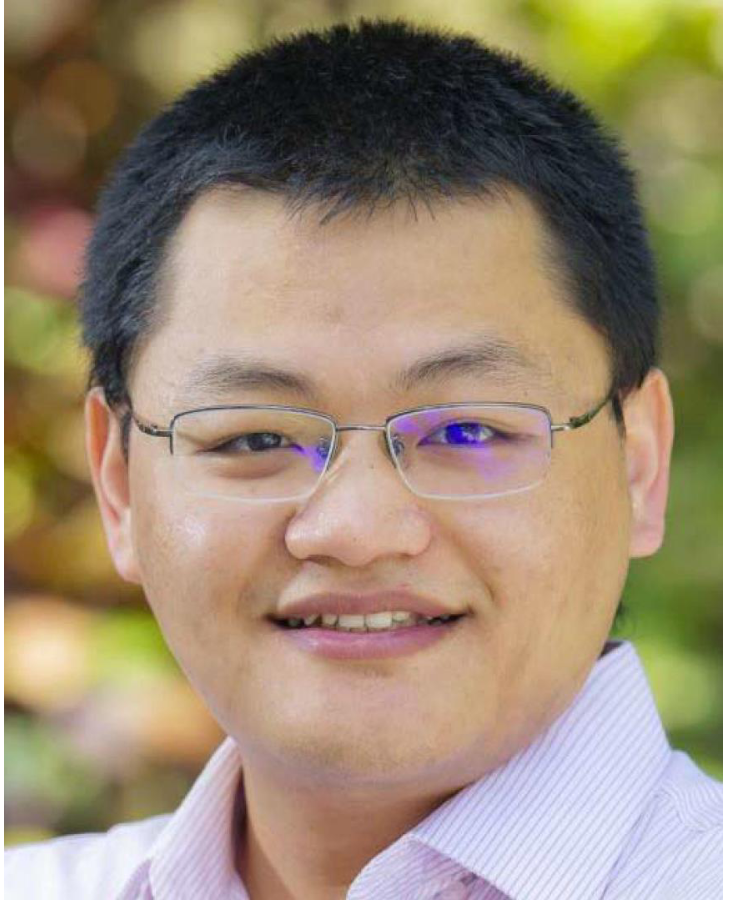}}]{Shiqi Wang} (Senior Member, IEEE) received the Ph.D. degree in computer application technology from Peking University in 2014. He is currently an Professor with the Department of Computer Science, City University of Hong Kong. He has proposed more than 70 technical proposals to ISO/MPEG, ITU-T, and AVS standards. He authored or coauthored more than 300 refereed journal articles/conference papers, including more than 100 IEEE Transactions. His research interests include video compression, image/video quality assessment, video coding for machine, and semantic communication. He received the Best Paper Award from IEEE VCIP 2019, ICME 2019, IEEE Multimedia 2018, and PCM 2017. His coauthored article received the Best Student Paper Award in the IEEE ICIP 2018. He served or serves as an Associate Editor for IEEE TIP, TCSVT, TMM, TCyber, Access, and APSIPA Transactions on Signal and Information Processing.
\end{IEEEbiography}

\begin{IEEEbiography}[{\includegraphics[width=1in,height=1.25in,clip,keepaspectratio]{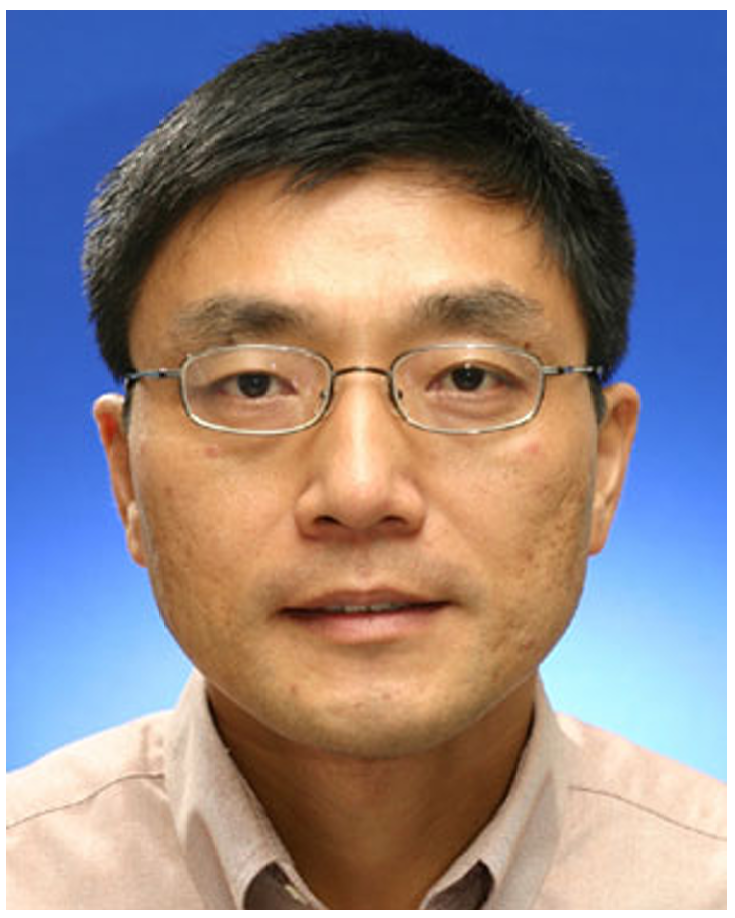}}]{Weisi Lin} (Fellow, IEEE) received the Ph.D. degree from the King’s College, University of London, U.K. He is currently a Professor with the College of Computing and Data Science, Nanyang Technological University. His areas of expertise include image processing, perceptual signal modeling, video compression, and multimedia communication, in which he has published over 200 journal articles, over 230 conference papers, filed seven patents, and authored two books. He has been an invited/panelist/keynote/tutorial speaker at over 20 international conferences. He is a fellow of IET and an Honorary Fellow of Singapore Institute of Engineering Technologists. He has been the Technical Program Chair of IEEE ICME 2013, PCM 2012, QoMEX 2014, and IEEE VCIP 2017. He has been an Associate Editor of IEEE TRANSACTIONS ON IMAGE PROCESSING, IEEE TRANSACTIONS ON CIRCUITS AND SYSTEMS FOR VIDEO TECHNOLOGY, IEEE TRANSACTIONS ON MULTIMEDIA, and IEEE SIGNAL PROCESSING LETTERS. He was a Distinguished Lecturer of Asia-Pacific Signal and Information Processing Association (APSIPA) from 2012 to 2013 and the IEEE Circuits and Systems Society from 2016 to 2017.
\end{IEEEbiography}

\end{document}